\pdfoutput=1
\RequirePackage[T1]{fontenc}
\documentclass[12pt]{article}

\usepackage[T1]{fontenc}
\usepackage[utf8]{inputenc}
\usepackage{microtype}

\usepackage{float}
\usepackage{multicol}
\usepackage{bm}
\usepackage{booktabs}
\usepackage{siunitx}
\usepackage{multirow}
\usepackage{amsmath,amssymb,amsfonts,amsthm,bbm}
\newtheorem{proposition}{Proposition}
\usepackage{mathtools}
\usepackage{graphicx}
\usepackage{textcomp}
\usepackage[table]{xcolor}
\usepackage{nicefrac}
\usepackage{braket}
\usepackage{adjustbox}
\usepackage{rotating}
\usepackage[percent]{overpic}
\usepackage{longtable}
\usepackage{algorithm}
\usepackage{algorithmicx}
\usepackage{algpseudocode}
\usepackage{thmtools,thm-restate}
\usepackage{enumitem}
\usepackage{authblk}
\usepackage[pagebackref=true,breaklinks=true,colorlinks,hyperfootnotes=false]{hyperref}
\hypersetup{
  colorlinks,
  citecolor=citeblue,
  linkcolor=firebrick,
  urlcolor=firebrick
  }
\usepackage[nameinlink,capitalize,noabbrev]{cleveref}

\usepackage{fullpage}
\usepackage[sort]{natbib}
\usepackage{orcidlink}
\usepackage{fancyhdr}
\newcommand{\runningtitle}{Gated QKAN-FWP: Scalable Quantum-inspired Sequence Learning}
\def\BibTeX{{\rm B\kern-.05em{\sc i\kern-.025em b}\kern-.08em
    T\kern-.1667em\lower.7ex\hbox{E}\kern-.125emX}}

\DeclareFontFamily{OT1}{pzc}{}
\DeclareFontShape{OT1}{pzc}{m}{it}{<-> s * [1.10] pzcmi7t}{}
\DeclareMathAlphabet{\pzccal}{OT1}{pzc}{m}{it}

\definecolor{lightyellow}{rgb}{1.0, 0.95, 0.7}
\definecolor{Blue}{rgb}{0, 0, 0.8}
\definecolor{blue}{rgb}{0,0,1}
\definecolor{mydarkblue}{rgb}{0,0.08,0.45}
\definecolor{mydarkblue2}{rgb}{0.133, 0.133, 0.698}
\definecolor{echodrk}{HTML}{0099cc}
\definecolor{mymauve}{rgb}{0.58,0,0.82}
\definecolor{darkgreen}{rgb}{0,0.40,0}
\definecolor{firebrick}{rgb}{0.698,0.133,0.133}
\definecolor{midnightblue}{rgb}{0.1,0.1,0.44}
\definecolor{citeblue}{RGB}{0, 113, 188}
\definecolor{oxfordblue}{rgb}{0.0,0.13,0.28}
\definecolor{prussianblue}{rgb}{0.0,0.19,0.33}
\definecolor{coolteal}{rgb}{0, 0.45, 0.45}
\definecolor{olive}{rgb}{0.1, 0.3, 0}
\definecolor{mypurple}{rgb}{0.5,0,0.5}
\definecolor{almond}{rgb}{0.94, 0.87, 0.8}

\definecolor{blue_ampEncoding}{HTML}{DAE8FC}
\definecolor{green_encoder}{HTML}{D5E8D4}
\definecolor{purple_decoder}{HTML}{E1D5E7}
\definecolor{yellow_measure}{HTML}{FFF2CC}
\definecolor{gray_block}{HTML}{F5F5F5}
\definecolor{pink_dru}{HTML}{FAD9D5}
\definecolor{orange_v}{HTML}{FAD7AC}
\definecolor{Lightblue}{HTML}{E7F4FC}

\DeclareDocumentCommand \norm { o m }{{\lVert #2 \rVert_#1}}

\newtheorem{lemma}{Lemma}[section]

\declaretheoremstyle[%
  spaceabove=10pt,%
  spacebelow=2pt,%
  headfont=\normalfont\itshape,%
  postheadspace=0em,%
  qed=%
]{prfstyle}

\title{\LARGE \rmfamily \bfseries
Gated QKAN-FWP:\\ Scalable Quantum-inspired Sequence Learning}

\author[1,2,\orcidlink{0009-0001-8342-2481}]{Kuo-Chung Peng\textsuperscript{*,}}
\author[3,\orcidlink{0000-0003-0114-4826}]{Samuel Yen-Chi Chen\textsuperscript{*,}}
\author[1,4,5,\orcidlink{0009-0005-1134-4962}]{Jiun-Cheng Jiang}
\author[6,\orcidlink{0000-0002-5437-5188}]{Chen-Yu Liu}
\author[7,\orcidlink{0000-0002-6770-0285}]{En-Jui Kuo}
\author[4,\orcidlink{0009-0001-0323-3382}]{Yun-Yuan Wang}
\author[8,\orcidlink{0000-0002-2851-4260}]{Prayag Tiwari}
\author[9,\orcidlink{0000-0003-0555-9543}]{Andrea Ceschini}
\author[10,\orcidlink{0000-0003-0807-0217}]{Chi-Sheng Chen}
\author[2,11,\orcidlink{0009-0004-7221-3854}]{Yu-Chao Hsu}
\author[1,2,\orcidlink{0009-0002-4383-0453}]{Chun-Hua Lin}
\author[2,\orcidlink{0000-0002-1993-1863}]{Tai-Yue Li}
\author[9,\orcidlink{0000-0002-4371-5925}]{Antonello Rosato}
\author[9,\orcidlink{0000-0002-9876-1494}]{Massimo Panella}
\author[12,\orcidlink{0000-0002-4958-9237}]{Simon See}
\author[13,\orcidlink{0000-0002-4402-7710}]{Saif Al-Kuwari}
\author[13,\orcidlink{0000-0002-6575-7034}]{Kuan-Cheng Chen\textsuperscript{$\dagger$,}}
\author[2,\orcidlink{0000-0001-8139-6809}]{Nan-Yow Chen\textsuperscript{$\dagger$,}}
\author[1,5,6,14,\orcidlink{0000-0001-8117-5846}]{Hsi-Sheng Goan\textsuperscript{$\dagger$,}}

\affil[1]{Department of Physics and Center for Theoretical Physics, National Taiwan University, Taipei, Taiwan}
\affil[2]{National Center for High-Performance Computing, National Institutes of Applied Research, Hsinchu, Taiwan}
\affil[3]{Wells Fargo, New York, NY, USA}
\affil[4]{NVIDIA AI Technology Center, NVIDIA Corp., Taipei, Taiwan}
\affil[5]{Center for Quantum Science and Engineering, National Taiwan University, Taipei, Taiwan}
\affil[6]{Graduate Institute of Applied Physics, National Taiwan University, Taipei, Taiwan}
\affil[7]{Department of Electrophysics, National Yang Ming Chiao Tung University, Hsinchu, Taiwan}
\affil[8]{School of Information Technology, Halmstad University, Sweden}
\affil[9]{Department of Information Engineering, Electronics and Telecommunications (DIET), University of Rome ``La Sapienza'', Rome, Italy}
\affil[10]{Beth Israel Deaconess Medical Center \& Harvard Medical School, Boston, MA, USA}
\affil[11]{Cross College Elite Program, National Cheng Kung University, Tainan, Taiwan}
\affil[12]{NVIDIA AI Technology Center, NVIDIA Corp., Singapore, Singapore}
\affil[13]{Qatar Center for Quantum Computing, College of Science and Engineering, Hamad Bin Khalifa University, Doha, Qatar}
\affil[14]{Physics Division, National Center for Theoretical Sciences, Taipei, Taiwan}
\affil[*]{These authors contributed equally to this work.}
\affil[$\dagger$]{Correspondence to: 
\href{mailto:kchen@hbku.edu.qa}{\texttt{kchen@hbku.edu.qa}},
\href{mailto:nanyow@nchc.narl.org.tw}{\texttt{nanyow@nchc.narl.org.tw}},
and \href{mailto:goan@phys.ntu.edu.tw}{\texttt{goan@phys.ntu.edu.tw}}.
}
\date{\today}

\begin{document}
\maketitle

\begingroup
\renewcommand\thefootnote{}
\footnotetext{The views expressed in this article are those of the authors and do not represent the views of Wells Fargo. This article is for informational purposes only. Nothing contained in this article should be construed as investment advice. Wells Fargo makes no express or implied warranties and expressly disclaims all legal, tax, and accounting implications related to this article.}
\endgroup

\clearpage

\begin{abstract}
Fast Weight Programmers (FWPs) encode temporal dependencies through dynamically updated parameters rather than recurrent hidden states. 
Quantum FWPs (QFWPs) extend this idea with variational quantum circuits (VQCs), but existing implementations rely on multi-qubit architectures that are difficult to scale on noisy intermediate-scale quantum (NISQ) devices and expensive to simulate classically. 
We propose gated QKAN-FWP, a fast-weight framework that integrates FWP with Quantum-inspired Kolmogorov–Arnold Network (QKAN) using single-qubit data re-uploading circuits as learnable nonlinear
activation, known as DatA Re-Uploading ActivatioN (DARUAN). 
We further introduce a scalar-gated fast-weight update rule that stabilizes parameter evolution, supported by a theoretical analysis of its adaptive memory kernel, geometric boundedness, and parallelizable gradient paths. 
We evaluate the framework across time-series benchmarks, MiniGrid reinforcement learning, and highlight real-world solar cycle forecasting as our main practical result. 
In the long-horizon setting with 528-month input window and 132-month forecast horizon, our 12.5k-parameter model achieves lower scaled Mean Square Error (MSE), peak amplitude error, and peak timing error than a suite of classical recurrent baselines with up to 13× more parameters, including Long Short-Term Memory (LSTM) networks (25.9k–89.1k parameters), WaveNet-LSTM (167k), Vanilla recurrent neural network (11.5k), and a Modified Echo State Network (132k).
To validate NISQ compatibility, we further deploy the trained fast programmer on IonQ and IBM Quantum processors, recovering forecasting accuracy within 0.1\% relative MSE of the noiseless simulator at 1024 shots. 
These results position gated QKAN-FWP as a scalable, parameter-efficient, and NISQ-compatible approach to quantum-inspired sequence modeling.

\vspace{1em}
\end{abstract}

\section{Introduction}
Modeling long-range temporal dependencies remains a central challenge in sequence learning and sequential decision making \cite{lin2025segrnn, liu2025typhoon, chen2024qeegnet}. In quantum machine learning (QML), this challenge is amplified by noisy intermediate-scale quantum (NISQ) hardware limitations \cite{preskill2018NISQ}. Consequently, deep, highly entangled quantum neural networks (QNNs) are difficult to execute reliably \cite{abbas2023quantum_review}, costly to simulate \cite{cerezo2025BP_and_simulability}, and hard to train \cite{mcclean2018barren,larocca2025barren}, especially within recurrent or long-horizon pipelines \cite{cerezo2021VQA, cerezo2022challenges,babbush2025grandchallenge}.
While hybrid variational quantum algorithms (VQAs) \cite{bharti2022nisqQA} have achieved breakthroughs in static domains like classification \cite{ bokhan2022multiclass, liu2025quantum, liu2025cls_hadmard, liu2025yomo, cong2019qConv, jing2025classH, chen2025qasadefi, chen2026qeegnet2,stein2022quclassi}, generative modeling \cite{huang2025generative, stein2021qugan, lloyd2018quantum, chen2025qrldiffusion, chen2025quantumnlg} and mathematical problem-solving \cite{kyriienko2021solvingPDE,patti2022multibasis_maxcut}, extending them to sequential frameworks poses a severe computational bottleneck. Quantum recurrent neural networks (QRNNs) require repeated circuit evaluations and backpropagation through time (BPTT) alongside expensive quantum gradient estimation \cite{wierichs2022PSR, abbas2023newBP_for_quantum}. As sequence length (window-size) grows, this training cost becomes prohibitive \cite{bausch2020recurrentQNN}.
Quantum Fast Weight Programmers (QFWPs) \cite{chen2024qfwp} mitigate this burden by replacing hidden-state dynamics with parameter dynamics. In QFWP, a classical slow programmer generates the parameters of a fast quantum model at each time step, thereby avoiding explicit quantum gradient computation inside a recurrent loop. However, existing QFWPs still rely on multi-qubit circuits, limiting practical scalability in the NISQ era. Recognizing these limitations, we shift our focus to a \textit{quantum-inspired} paradigm that inherently bypasses the hardware constraints. We propose \textbf{gated QKAN-FWP}, integrating Quantum-inspired Kolmogorov--Arnold Network (QKAN) \cite{jiang2025qkan} into the fast-weight programming framework. QKAN utilizes single-qubit data re-uploading circuits as learnable nonlinear activations known as DatA Re-Uploading ActivatioN (DARUAN) \cite{jiang2025qkan,Schuld_2021, perez2020data}, circumventing multi-qubit entanglement to provide expressive, hardware-friendly, and simulation-efficient modeling \cite{jiang2025qkan}. To further stabilize parameter evolution, we introduce a gated fast-weight update rule. By completely avoiding multi-qubit entanglement bottlenecks, our architecture bridges the gap between quantum concepts and classical execution. Therefore, we emphasize evaluating our model against classical baselines on practical tasks, specifically real-world long-horizon direct multi-step forecasting—a capability that remains largely out of reach for prior quantum models constrained by NISQ limits.

The main contributions of this work are as follows:
\begin{enumerate}[leftmargin=*]
    \item We propose gated QKAN-FWP, a quantum-inspired framework integrating QKAN modules with fast-weight programming for efficient sequence modeling.
    
    \item We introduce a scalar-gated fast-weight mechanism that adaptively balances memory retention and new updates, with theoretical support through adaptive memory kernels, geometric bounds, and a parallelizable unrolled recursion that yields shallower gradient paths than general recurrent neural networks (RNNs).
    
    \item We demonstrate strong empirical performance on real-world multi-step solar cycle forecasting, where our 12.5k-parameter model outperforms classical recurrent baselines spanning 11.5k to 167k parameters (up to 13× our model's size). We also evaluate comprehensively across time-series benchmarks, and MiniGrid reinforcement learning (RL).
    
    \item We validate NISQ compatibility by executing the trained fast programmer on two quantum processing units (QPUs), recovering forecasting performance within $10^{-3}$ relative Mean Square Error (MSE) of the noiseless simulator.
\end{enumerate}

\section{Related Work}

\paragraph{Quantum sequence modeling and reinforcement learning}
For sequential modeling, QRNNs and Quantum Long Short-Term Memory (QLSTM) variants have been introduced to adapt quantum neural architectures for temporally dependent tasks \cite{bausch2020recurrentQNN, ceschini2024QRU, wu2026QRU,chen2022qlstm,chen2022reservoir,hsu2024kernel,chen2025toward}. Parallel to these developments, early quantum reinforcement learning (QRL) formulations assumed fully quantum environments \cite{dong2008qrl}. Recent approaches instead utilize variational quantum circuits (VQCs) in classical environments with discrete or continuous observations \cite{chen2025qtrl, skolik2022quantum,chen2020variational,lockwood2020reinforcement, patel2024curriculum, dai2025qrl}. Furthermore, to overcome the limitations of partially observable environments, where agents must inherently track historical states, recent works have integrated QRNNs into RL policies \cite{chen2023QDRL, chen2024efficient}.

\paragraph{Fast-weight programming and quantum extensions}
Fast Weight Programmers (FWPs) \cite{schmidhuber1992learning,schmidhuber1993reducing} replace recurrent hidden-state evolution with dynamical evolution in parameter space. A slow network updates the parameters of a fast network, enabling memory-like behavior without explicit recurrence. Subsequent classical work has combined FWPs with RNNs \cite{schlag2017gated} and established analogies to linear Transformers \cite{schlag2021transformerfwp, irie2021transformerfwp}. QFWPs extend this paradigm by utilizing a parameterized quantum circuit as the fast programmer \cite{chen2024qfwp}. In QFWP, a classical slow network generates quantum circuit parameters on the fly, eliminating explicit quantum gradient computation inside the temporal loop. To further reduce the parameter size, QT-QFWP \cite{liu2025qtfwp} uses a generative QNN to synthesize the slow programmer's weights, leveraging quantum expressivity to address the scalability bottlenecks of classical slow networks.

\paragraph{KAN and QKAN architectures}
Kolmogorov--Arnold Networks (KANs) replace fixed activation functions in multilayer perceptrons (MLPs) with learnable univariate functions, yielding interpretable and parameter-efficient nonlinear modeling \cite{liu2025kan,kundu2024kanqas,liu2025kolmogorov,lee2026kano, somvanshi2025surveyKAN,noorizadegan2026practitionersguidekolmogorovarnoldnetworks,yang2025kolmogorovarnold}. This efficiency has motivated adaptation for temporal sequence modeling tasks \cite{huang2025timekan,jarraya2025soh,vaca2024kolmogorov,xu2024kolmogorov,livieris2024c,yamak2025kolmogorov}. QKAN extends the KAN architecture by implementing the edge functions with DARUAN \cite{jiang2025qkan}. The resulting quantum-inspired activations offer rich spectral expressivity while remaining lightweight and easily simulable. Prior work \cite{hsu2025qkan} embeds QKAN inside the gates of a Long Short-Term Memory (LSTM) cell to form QKAN-LSTM. Because its computation depends on the recurrent hidden state $h_{t-1}$, execution across the time dimension remains strictly sequential, and BPTT must traverse a chain of $T$ hidden-state Jacobians. In contrast, we deploy QKAN within a fast-weight programmer. Since the fast-parameter updates $\Delta W_k$ depend solely on the input $x_k$ rather than previous parameters $W_{k-1}$, we bypass the recurrent bottleneck, yielding shallower gradient paths (\cref{subsec:gated_theory}). This positions QKAN as a building block for fast-weight programming, distinct from its nonlinear recurrent-gate role in \cite{hsu2025qkan}.

\section{Preliminaries}
\subsection{Quantum-inspired Kolmogorov--Arnold Networks and Hybrid QKAN architecture}
\label{subsec:qkan}
QKAN extends the KAN paradigm by replacing classical spline-based edge functions with quantum-inspired univariate functions realized by DARUAN \cite{jiang2025qkan,liu2025kan}. For an input \(x\), each activation is defined as 
\begin{equation}
\phi_{\theta}(x)=\langle 0|U^{\dagger}(x;\theta) \hat{O} U(x;\theta)|0\rangle,
\end{equation}
where \(U(x;\theta)\) is a parameterized single-qubit data re-uploading unitary and \(\hat{O}\) is a measurement observable. Repeating data re-uploading induces a rich Fourier spectrum, enabling QKAN to represent highly nonlinear mappings with relatively few trainable parameters \cite{jiang2025qkan}. QKAN scales efficiently on CPUs, GPUs and HPC clusters, a property empirically validated by its use in large language models (LLMs) \cite{jiang2025qkan}. Beyond classical simulation efficiency, the strictly single-qubit paradigm is compatible with current NISQ hardware, where state-of-the-art platforms achieve single-qubit error rates of $10^{-5}$ -- $10^{-7}$ \cite{wu2025simultaneoushighfidelitysinglequbitgates,rower2024suppressing,smith2025single}. In~\cref{subsec:hardware}, we confirm this compatibility by deploying our trained model on IonQ and IBM QPUs.

We adopt the Hybrid QKAN (HQKAN) instantiation of the Jiang--Huang--Chen--Goan network (JHCG Net) first introduced in \cite{jiang2025qkan}. HQKAN has an encoder--processor--decoder structure: a classical encoder maps the input into a latent representation, a QKAN block performs nonlinear transformation in the latent space, and a decoder maps the transformed features to the output as illustrated in \cref{fig:hqkan}. Within our framework, HQKAN acts as a drop-in programmer network. When used as the slow programmer, it generates fast-parameter updates from the current input. When used as the fast programmer, its DARUAN parameters are dynamically updated by the slow programmer.

\begin{figure}[!t]
\centering
\includegraphics[width=\textwidth,trim={4cm 0.5cm 4cm 0cm}, clip]{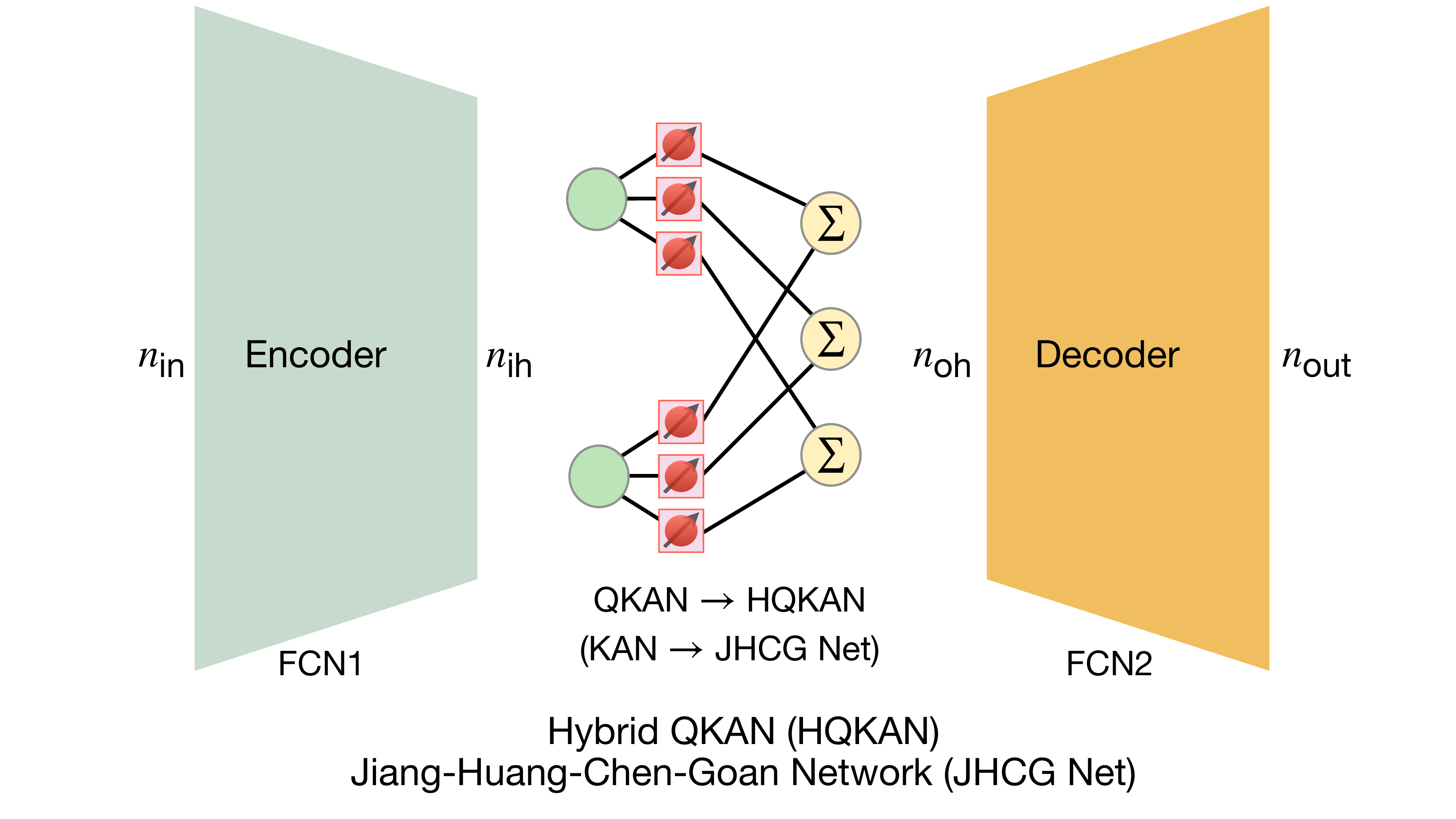}
\caption{\textbf{HQKAN programmer architecture adapted from \cite{jiang2025qkan}}. The model consists of a classical encoder, a latent QKAN processor, and a decoder. In this paper, HQKAN is used as a compact nonlinear programmer network inside the fast-weight framework.}
\label{fig:hqkan}
\end{figure}

\subsection{Fast-weight programming}
\label{subsec:FWP}
\paragraph{FWP}
FWPs model sequential data through dynamical evolution in parameter space rather than hidden-state recurrence. Let \(x_t\) be the input at time step \(t\), \(S(\cdot)\) the slow programmer, and \(F(\cdot;\,W_t)\) the fast programmer with time-dependent parameters \(W_t\). The fast network produces $y_t = F(x_t; W_t),$
while the slow programmer generates an update
$\Delta W_t = S(x_t).$
The fast parameters then evolve according to
\begin{equation}
W_{t+1} = W_t + \Delta W_t.
\label{eq:ungated}
\end{equation}
Temporal dependencies are therefore encoded in the trajectory of the fast parameters \(\{W_t\}\).

\paragraph{QFWP}
In QFWP \cite{chen2024qfwp}, the fast programmer is a VQC. A classical encoder maps \(x_t\) to two vectors \(L_t \in \mathbb{R}^{l}\) and \(Q_t \in \mathbb{R}^{n}\), corresponding to the number of circuit layers \(l\) and qubits \(n\), respectively. The update is formed as an outer product
\begin{equation*}
\Delta \Theta_t = L_t \otimes Q_t, \qquad (\Delta \Theta_t)_{ij}=L_{t,i}Q_{t,j},
\end{equation*}
which updates the quantum parameters \(\Theta_t \in \mathbb{R}^{l\times n}\): $\Theta_{t+1} = \Theta_t + \Delta \Theta_t.$
The model output is the expectation value of the fast VQC, 
\begin{equation*}
y_t = \langle 0|U^{\dagger}(\Theta_t,x_t)\,\hat{O}\,U(\Theta_t,x_t)|0\rangle.
\end{equation*}

\section{Methods}
\subsection{Gated fast-weight update}
A central contribution of this work is a gated update rule that stabilizes the evolution of the fast parameters. At each time step, the slow programmer outputs the update components together with a scalar gate \(g_t \in [0,1]\) through a sigmoid nonlinearity. The gate interpolates between the previously stored fast parameters and the newly generated update. This mechanism is mathematically analogous to the ``write-strength'' utilized in linear transformers \cite{schlag2021transformerfwp,yang2023gated}, where a data-dependent weight adaptively blends previous attention values with new updates. While in \cite{schlag2017gated}, a gated fast-weight architecture was introduced for RNNs using an element-wise matrix gate, our framework introduces a scalar gating mechanism. This scalar approach ensures uniform parameter scaling, making it parameter-efficient and naturally scalable. For the fast parameters $W_t$, our gated update is formulated as:
\begin{equation}
\label{eq:gated}
W_{t+1} = g_t W_t + (1-g_t)\Delta W_t,
\qquad g_t \in [0,1].
\end{equation}
Intuitively, when \(g_t \to 1\), the model retains its previously stored fast parameters, whereas \(g_t \to 0\) forces the model to rely entirely on the newly generated update. We analyze these dynamics theoretically in~\cref{subsec:gated_theory}. 

\begin{figure}[!t]
\centering
\includegraphics[width=\textwidth, trim={0cm 0cm 0cm 0cm}, clip]{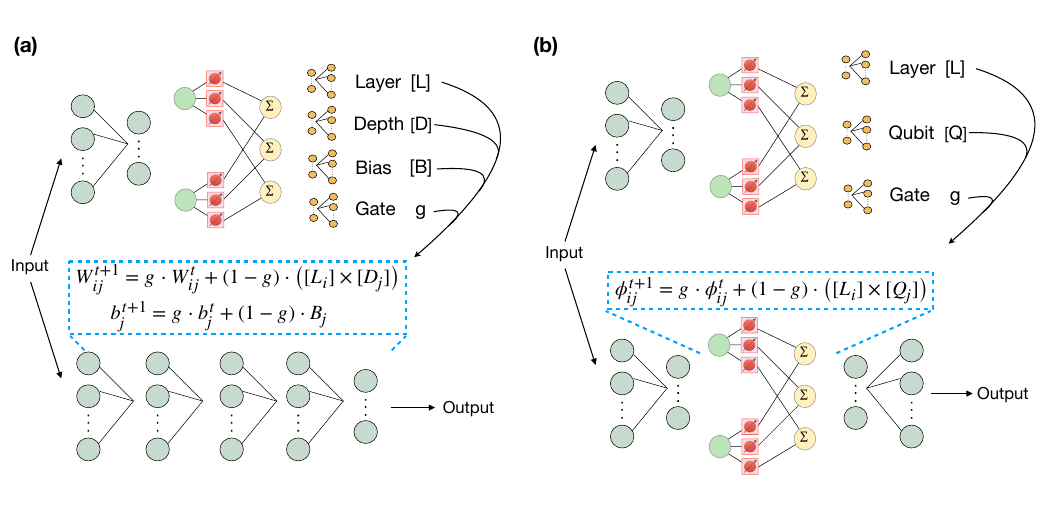}

\caption{\textbf{Architectures of the proposed gated fast-weight programmers.} \textbf{(a) GQKAN-FWP:} An HQKAN slow programmer dynamically generates the parameters of a classical linear fast programmer following a gated update rule. \textbf{(b) GQKAN-QKANFWP:} Both programmers utilize HQKAN, with the slow programmer generating the DARUAN parameters for the fast module under the same gated mechanism.}
\label{fig:model_architecture}

\end{figure}

\subsection{Model variants}

To systematically evaluate our framework, we investigate the ablation variants summarized in \cref{tab:model_variants}. For variants utilizing a classical fast programmer, the slow programmer produces update vectors $L_t \in \mathbb{R}^{l}$, $D_t \in \mathbb{R}^{n}$, and $B_t \in \mathbb{R}^{n}$. In the ungated setting (e.g., FWP), the fast-weight $W_t$ and bias $b_t$ are computed as: 
\begin{equation*}
W_{t+1} = W_t + L_t \otimes D_t \quad \text{and} \quad b_{t+1} = b_t + B_t,
\end{equation*}
yielding the output:
\begin{equation*}
y_t = x_t W_t + b_t.
\end{equation*}
Conversely, the gated variants (e.g., G-FWP, GQKAN-FWP) update these parameters according to \cref{eq:gated}. For models employing HQKAN as the fast programmer (e.g., G-QKANFWP, GQKAN-QKANFWP), let $\phi_t$ denote the fast-parameters. At each time step, the slow programmer generates the parameter update $\Delta \phi_t$ alongside the gate $g_t$. The fast parameters then evolve via the gated mechanism:
\begin{equation*}
\phi_{t+1} = g_t \phi_t + (1-g_t)\Delta \phi_t,
\end{equation*}
and the prediction is produced by the fast HQKAN programmer:
\begin{equation*}
y_t = f_{\mathrm{HQKAN}}(x_t;\phi_t).
\end{equation*}

Structural illustrations of GQKAN-FWP and GQKAN-QKANFWP are presented in \cref{fig:model_architecture}(a) and (b), respectively.

\begin{table}[!t]
\centering
\caption{Summary of model variants and their architectural components.}
\label{tab:model_variants}

\resizebox{\textwidth}{!}{
\begin{tabular}{lccc}
\toprule
Model & Gated & Slow Programmer & Fast Programmer \\
\midrule
\multicolumn{4}{l}{\textit{Baselines}} \\
FWP & No & Classical & Classical \\
QFWP & No & Classical & VQC \\
\midrule
\multicolumn{4}{l}{\textit{Proposed Models}} \\
G-FWP & Yes & Classical & Classical \\
G-QFWP & Yes & Classical & VQC \\
GQKAN-QFWP & Yes & HQKAN & VQC \\
GQKAN-FWP (\cref{fig:model_architecture}(a)) & Yes & HQKAN & Classical \\
G-QKANFWP & Yes & Classical & HQKAN \\
GQKAN-QKANFWP (\cref{fig:model_architecture}(b)) & Yes & HQKAN & HQKAN \\
\bottomrule
\end{tabular}
}

\end{table}

\section{Theoretical Analysis}
\label{subsec:gated_theory}

We provide a theoretical interpretation of the gated fast-weight update which is the mechanism introduced in \cref{eq:gated}. This update is motivated as a way to interpolate between the previously stored fast parameters and the newly generated update. The same analysis below also applies to the gated variants, after replacing \(W_t\) by the corresponding fast parameters.  
For comparison, the ungated fast-weight recursion is given by \cref{eq:ungated}, which accumulates all past updates additively.  
\paragraph{Unrolled form and adaptive memory kernel}
By recursively expanding \cref{eq:gated}, we obtain
\begin{equation}
W_{t+1}
=
\left(\prod_{s=1}^{t} g_s\right) W_1
+
\sum_{k=1}^{t}
(1-g_k)\left(\prod_{s=k+1}^{t} g_s\right)\Delta W_k .
\label{eq:gated_unroll}
\end{equation}
Therefore, the current fast parameters are a weighted aggregation of all past proposed fast states \(\{\Delta W_k\}_{k=1}^t\), together with a decayed contribution from the initialization \(W_1\). Define
\begin{equation}
\beta_{0,t} := \prod_{s=1}^{t} g_s,
\quad
\beta_{k,t} := (1-g_k)\prod_{s=k+1}^{t} g_s,
\quad k=1,\dots,t.
\label{eq:beta_def}
\end{equation}
Since \(g_t \in [0,1]\), we have \(\beta_{k,t}\ge 0\) for all \(k\),
and one may verify by induction that
\begin{equation}
\beta_{0,t} + \sum_{k=1}^{t} \beta_{k,t} = 1.
\label{eq:beta_sum}
\end{equation}
Hence \cref{eq:gated_unroll} can be written as
\begin{equation}
W_{t+1} = \beta_{0,t} W_1 + \sum_{k=1}^{t} \beta_{k,t}\Delta W_k,
\label{eq:gated_convex}
\end{equation}
which shows that the gated dynamics implement an input-dependent temporal kernel in parameter space. This interpretation highlights a key distinction from the ungated update in \cref{eq:ungated}, where every past update enters with a coefficient of \(1\) lacking a forgetting mechanism. In contrast, under \cref{eq:gated}, the contribution of \(\Delta W_k\) at time \(t+1\) is weighted by
\begin{equation}
\beta_{k,t} = (1-g_k)\prod_{s=k+1}^{t} g_s,
\label{eq:memory_coeff}
\end{equation}
which decays according to the subsequent gates. Thus, the gated recursion supports both long-memory and short-memory behavior: when the subsequent gates remain close to \(1\), older proposals are retained for many steps; when the gates are small, older proposals are rapidly forgotten. In the special case \(g_t \equiv g\), \cref{eq:memory_coeff} reduces to
\begin{equation}
\beta_{k,t} = (1-g)g^{\,t-k},
\label{eq:exp_kernel}
\end{equation}
which is an exponential memory kernel.

\paragraph{Geometric boundedness}
A second useful consequence of \cref{eq:gated_convex} is that \(W_{t+1}\) lies in the convex hull of the set
\begin{equation}
\mathcal{S}_t := \{W_1,\Delta W_1,\dots,\Delta W_t\}.
\label{eq:St_def}
\end{equation}
Therefore, for any norm \(\|\cdot\|\),
\begin{align}
\|W_{t+1}\|
&\le
\beta_{0,t}\|W_1\| + \sum_{k=1}^{t}\beta_{k,t}\|\Delta W_k\| \nonumber \\
&\le
\max\bigl\{\|W_1\|,\|\Delta W_1\|,\dots,\|\Delta W_t\|\bigr\}.
\label{eq:gated_bound}
\end{align}
This provides a simple geometric boundedness property that the gated update cannot move the fast parameters outside the convex hull generated by the initialization and the historical proposals.
By contrast, the ungated recursion in \cref{eq:ungated} admits only the crude estimate
\begin{equation}
\|W_{t+1}\|
\le
\|W_1\| + \sum_{k=1}^{t}\|\Delta W_k\|,
\label{eq:ungated_bound}
\end{equation}
which can grow linearly with the sequence length (window-size) in the worst case. Hence, whereas the ungated dynamics perform unconstrained additive accumulation, the gated dynamics replace this behavior by adaptive convex aggregation, yielding a built-in forgetting mechanism together with a norm bound controlled by the historical proposals.

\paragraph{Parallelizable parameter evolution and shallow gradient path}
A further consequence of the unrolled form~\cref{eq:gated_unroll} is 
computational. Since the slow programmer produces $\Delta W_k$ and gate 
$g_k$ directly from $x_k$ alone, independent of $W_{k-1}$, the sets 
$\{(\Delta W_k, g_k)\}_{k=1}^{T}$ for a sequence of length $T$ can be 
computed in a single parallel pass.

Observe that~\cref{eq:gated} is affine in $W_t$ with a \emph{scalar} 
multiplier. Writing $a_t := g_t \in [0,1]$ and 
$b_t := (1-g_t)\Delta W_t$, the recursion becomes
\begin{equation}
\label{eq:affine_recursion}
W_{t+1} = a_t W_t + b_t.
\end{equation}
The pairs $(g_k, (1-g_k)\Delta W_k)$ compose under the associative rule
\begin{equation}
\label{eq:scan_composition}
(a', b') \circ (a, b) = (a'a,\; a'b + b'),
\end{equation}
so the trajectory $\{W_t\}_{t=1}^{T}$ can be resolved by a parallel 
prefix scan \cite{blelloch1990prefix, martin2018parallelizing} with 
\begin{equation}
\label{eq:scan_complexity}
O\!\left(\frac{T}{p} + \log p\right)
\end{equation}
scan time on $p$ processors \cite{blelloch1990prefix}, reducing to 
$O(\log T)$ depth when $p = \Theta(T)$, in contrast to the $\Omega(T)$ 
sequential depth \cite{martin2018parallelizing} of general nonlinear 
recurrent hidden-state evolution. Moreover, each $\Delta W_k$ factors 
through an independent forward pass of the slow programmer, so BPTT 
composes through products of scalar gates rather than a chain of $T$ 
dense hidden-state Jacobians as in QKAN-LSTM \cite{hsu2025qkan}.
 
\paragraph{Implication}
The above analyses suggest that the gate plays three complementary roles: it induces an adaptive memory kernel, guarantees geometric boundedness of the fast parameters, and preserves the parallel, hidden-state-free structure of the FWP recursion. Together, these properties help explain the empirically improved stability of the gated variants relative to their ungated counterparts.

\section{Experimental Results}
We evaluate the proposed framework on single-step time-series prediction, multi-step real-world forecasting and RL tasks. To ensure robust and unbiased evaluation, all models across every experiment are independently trained and tested over five random seeds. Furthermore, to provide a fair comparison of representational capacity, all quantum baselines are executed on classical simulators utilizing exact gradients computed via BPTT, without simulated hardware noise or finite measurement shots. All quantum-circuit simulation experiments are implemented using PennyLane \cite{bergholm2018pennylane}, PyTorch \cite{paszke2019pytorch}, and an open-source QKAN implementation adapted from \cite{jiang2025qkan_github}\footnote{Available at \href{https://github.com/Jim137/qkan}{https://github.com/Jim137/qkan}}. To accelerate the QKAN framework, we adopt the PyTorch-based efficient quantum-circuit solver, \texttt{FlashQKAN}, introduced in \cite{jiang2025qkan_github}. By representing each QKAN layer as a tensor network and leveraging \texttt{cuQuantum} \cite{10313722} to optimize the tensor-contraction path, while \texttt{cuTile} \cite{cutile} is used for fused operator execution and block tiling to improve GPU throughput. For the quantum hardware experiments in~\cref{subsec:hardware}, we execute the trained fast programmer on IonQ's \texttt{Forte-1} trapped-ion system \cite{chen2024ionq} using NVIDIA CUDA-Q \cite{kim2023cuda}—a unified programming platform enabling seamless access to QPUs across modalities—with Amazon Braket \cite{braket} as the access provider, and on the IBM Quantum superconducting Heron r3 processor \texttt{ibm\_aachen} \cite{ibmquantum2026} via Qiskit \cite{javadi2024qiskit}.

\subsection{Time-series prediction}
\label{subsec:time_series}
We evaluate the models on four benchmark datasets used in \cite{chen2024qfwp}---Damped Simple Harmonic Motion (SHM), the Bessel function, and Nonlinear Auto-Regressive Moving Average (NARMA5 and NARMA10)---and two additional datasets related to quantum dynamics: Delayed Quantum Control (DQC) and open quantum system Jaynes-Cummings (JC) dynamics. Across all tasks, we frame next-step prediction as a sequential modeling problem. Given an input sequence of sliding window-size (sequence-length) $N$ previous observations $[x_{t-N}, x_{t-N+1}, \dots, x_{t-1}]$, the model processes each element $x_\tau$ one at a time for $\tau \in [t-N, t-1]$. After processing the full sequence, the model outputs a prediction $y_t$ at the final time step, which is evaluated against the ground truth $x_t$ using MSE. Each dataset is normalized to the range $[-1, 1]$ and chronologically split into 80\% training and 20\% test data. Each model is trained for 50 epochs with a batch size of 4 and a learning rate of $1\times10^{-3}$. \cref{tab:num_parameter} summarizes each model's trainable parameter counts for \cref{subsec:time_series} and \cref{subsec:rl}.

We evaluate the models in two stages. Stage I fixes the input window-size to $N=16$ as an ablation study to rank all variants under a common setting. Stage II evaluates the top-performing models across variable input window-sizes $N\in\{8,16,32,64\}$ to test the models' capacity to retain memory and capture both short- and long-range temporal dependencies.

\subsubsection{Datasets}
\noindent\textbf{Damped SHM.}
Damped SHM is a standard benchmark for nonlinear function approximation. We model the angular velocity $\dot{\theta}$ of a damped pendulum governed by 
\begin{equation*}
\frac{d^2 \theta}{dt^2} + \frac{b}{m}\frac{d\theta}{dt} + \frac{g}{L}\sin\theta = 0,
\end{equation*}
where $g=9.81$, $b=0.15$, $L=1$, and $m=1$, with initial conditions $\theta(0)=0$ and $\dot{\theta}(0)=3$.

\noindent\textbf{Bessel function.}
Bessel functions arise in many physical applications, such as wave propagation and heat conduction in cylindrical geometries. The target is the second-order Bessel function of the first kind, $J_2(x)$, which satisfies 
\begin{equation*}
x^2 \frac{d^2 y}{dx^2} + x\frac{dy}{dx} + (x^2-\alpha^2)y = 0
\end{equation*}
with the series representation:
\begin{equation*}
J_{\alpha}(x)=\sum_{m=0}^{\infty}\frac{(-1)^m}{m!\,\Gamma(m+\alpha+1)}\left(\frac{x}{2}\right)^{2m+\alpha}.
\end{equation*}

\noindent\textbf{NARMA.}
We use the standard NARMA5 ($n_0=5$) and NARMA10 ($n_0=10$) benchmarks following \cite{chen2024qfwp} with $M=300$ timesteps 
generated from the recurrence: 
\begin{equation*}
y_{t+1} = \alpha y_t + \beta y_t \sum_{j=0}^{n_0-1} y_{t-j} + \gamma u_{t-n_0+1}u_t + \delta,
\end{equation*}
where $(\alpha,\beta,\gamma,\delta)=(0.3,0.05,1.5,0.1)$. The input sequence is:  
\begin{equation*} 
u_t = 0.1 \left[ \sin\left(\frac{2\pi \bar{\alpha} t}{T}\right) \sin\left(\frac{2\pi \bar{\beta} t}{T}\right) \sin\left(\frac{2\pi \bar{\gamma} t}{T}\right) + 1 \right],
\end{equation*}
where $(\bar{\alpha},\bar{\beta},\bar{\gamma},T)=(2.11,3.73,4.11,100)$. 

\noindent\textbf{DQC.}
To evaluate the model's capacity for long-term temporal dependencies, we consider a non-Markovian \cite{fang2018non-markovian} system of a two-level atom (qubit) coupled to a semi-infinite waveguide terminated by a mirror, inducing delayed quantum feedback via a bound state in the continuum \cite{calajo2019dqc,tufarelli2013dynamics}. Following \cite{chen2022qlstm}, we predict the output field intensity $x(t)$, modeled as a sequence of localized pulses with decaying amplitude: 
\begin{equation*}
x(t)=\sum_{n=0}^{10}\exp\left[-10(t-2n)^2\right]\exp\left(-\frac{t}{16}\right)
\end{equation*}
for $t \in [-2, 20]$. This formulation captures the structured, non-stationary nature of the delayed feedback response.

\noindent\textbf{Open Quantum System JC Dynamics.}
To incorporate realistic environmental noise, we simulate the dynamics of an open quantum system based on the JC model using CUDA-Q Dynamics, the open quantum system simulation backend of CUDA-Q \cite{kim2023cuda}. The system Hamiltonian is: \begin{equation*}
\mathcal{H} = \omega_c a^\dagger a + \omega_q \sigma_+ \sigma_- + g(\sigma_- a^\dagger + \sigma_+ a),
\end{equation*}
where $a^\dagger, a$ are bosonic creation and annihilation operators, $\sigma_+, \sigma_-$ are qubit raising and lowering operators, $\omega_c = \omega_q = 2\pi$ denotes the resonant cavity and qubit frequencies, and $g = \pi$ is the coupling strength. To capture non-unitary evolution, we apply a collapse operator $C = \sqrt{\gamma} a$ (decay rate $\gamma=0.05$) representing photon loss. While the coupling ratio $g/\omega = 0.5$ exceeds typical experimental values, it is chosen for simplicity and to yield a demanding benchmark curve that combines rapid oscillations with dissipative decay. The system is initialized with the qubit in its ground state and a single cavity photon ($\rho_0 = |g, 1\rangle\langle g, 1|$), the target signal is the qubit excitation expectation probability $\langle \sigma_+ \sigma_- \rangle(t)$, evaluated over 3,000 discrete time steps from $t=0$ to $t=50$.

\begin{table}[!t]
\centering
\caption{Trainable parameter counts for each model in the time-series prediction and reinforcement-learning experiments.}
\label{tab:num_parameter}

\resizebox{\textwidth}{!}{
\begin{tabular}{l S S}
\toprule
Model & {Time-series prediction (\cref{subsec:time_series})} & {Reinforcement learning(\cref{subsec:rl})} \\
\midrule
FWP & 128 & 2656\\
QFWP & 111 & 2530\\
G-FWP & 137 & 2665\\
G-QFWP & 120 & 2539\\
GQKAN-QFWP & 100 & 1801\\
GQKAN-FWP & 113 & 2605\\
G-QKANFWP & 116 & 1786\\
GQKAN-QKANFWP & 159 & 1114 \\
\bottomrule
\end{tabular}
}

\end{table}

\subsubsection{Stage I: fixed window-size evaluation}
\cref{tab:time_first_stage} reports the final test loss at input window-size $N=16$. HQKAN-based gated models provide better overall balance across datasets. In particular, GQKAN-QKANFWP attains the best result on three of the six datasets, while GQKAN-FWP and G-QKANFWP each rank among the top two in multiple tasks. In contrast, QFWP achieves the best result on NARMA10. We advance GQKAN-QKANFWP, G-QKANFWP, GQKAN-FWP and QFWP to Stage II.

\subsubsection{Stage II: variable window-size evaluation}
\cref{tab:bessel,tab:narma,tab:quantum} report the results, from which four trends emerge. First, GQKAN-QKANFWP exhibits the greatest robustness to varying $N$, attaining the lowest prediction error in 10 of the 24 settings. Second, G-QKANFWP dominates NARMA5 and NARMA10 at longer windows, indicating that an HQKAN fast programmer is well suited to long-range nonlinear dependencies. Third, GQKAN-FWP leads on the quantum-dynamics datasets (DQC and JC), where an HQKAN-based slow programmer captures delayed feedback and dissipative noise effectively. Fourth, while the QFWP attains an MSE of $2\times10^{-6}$ on NARMA10 at $N{=}16$, it degrades to $1.3\times10^{-4}$ at $N\in\{32,64\}$, a roughly $60\times$ collapse. Similar degradation trends are also shown in the quantum dynamics datasets~\cref{tab:quantum}.
Taken together, the results indicate that the gated QKAN-FWP variants yield the most favorable balance between accuracy and stability across window sizes. Among them, GQKAN-QKANFWP exhibits the most consistent behavior across all three dataset families: it is best on every window-size for the smooth-dynamics benchmarks, best or second-best on three of four window sizes for the quantum-dynamics benchmarks, and on the NARMA benchmarks remains close to the leading variants. This stability is consistent with the theoretical properties of the gated memory mechanism (\cref{subsec:gated_theory}) and the spectral expressivity of the HQKAN architecture (\cref{subsec:qkan}), and is precisely the property required for real-world forecasting, which motivates our selection of GQKAN-QKANFWP for the real-world forecasting study in \cref{subsec:multistep}.

\begin{table*}[!t]
\centering
\caption{Final test loss (MSE, mean \(\pm\) std over 5 seeds) with window-size \(N=16\). Best/second-best results are shown in \textbf{bold}/\underline{underlined}.}
\label{tab:time_first_stage}

\resizebox{\textwidth}{!}{
\begin{tabular}{lcccccc}
\toprule
Model & Bessel & Damped SHM & Narma5 & Narma10 & Delayed Quantum Control & Jaynes-Cummings \\
\midrule
FWP & 0.004616$\pm$0.001822 & 0.000140$\pm$0.000110 & 0.000208$\pm$0.000251 & 0.000138$\pm$0.000161 & 0.000140$\pm$0.000110 & 0.000699$\pm$0.001117 \\
QFWP & 0.003918$\pm$0.001110 & 0.003460$\pm$0.004777 & 0.000035$\pm$0.000017 & \textbf{0.000002$\pm$0.000002} & 0.003119$\pm$0.003364 & 0.010183$\pm$0.019995 \\
G-FWP & 0.000985$\pm$0.001626 & 0.000534$\pm$0.001003 & 0.000014$\pm$0.000009 & 0.000013$\pm$0.000005 & 0.000127$\pm$0.000063 & 0.000307$\pm$0.000258 \\
G-QFWP & 0.001384$\pm$0.001672 & 0.001682$\pm$0.001368 & \underline{0.000013$\pm$0.000010} & 0.000031$\pm$0.000036 & 0.000143$\pm$0.000038 & 0.000545$\pm$0.000270 \\
GQKAN-QFWP & \underline{0.000108$\pm$0.000182} & 0.000167$\pm$0.000052 & 0.000031$\pm$0.000012 & 0.000089$\pm$0.000049 & 0.000074$\pm$0.000069 & 0.000188$\pm$0.000264 \\
GQKAN-FWP & 0.000368$\pm$0.000688 & \underline{0.000059$\pm$0.000035} & 0.000036$\pm$0.000021 & 0.000071$\pm$0.000028 & \underline{0.000053$\pm$0.000077} & \textbf{0.000070$\pm$0.000074} \\
G-QKANFWP & 0.000661$\pm$0.001311 & 0.002287$\pm$0.001290 & \textbf{0.000012$\pm$0.000006} & \underline{0.000011$\pm$0.000008} & 0.000272$\pm$0.000212 & 0.000664$\pm$0.000041 \\
GQKAN-QKANFWP & \textbf{0.000011$\pm$0.000012} & \textbf{0.000036$\pm$0.000025} & 0.000028$\pm$0.000018 & 0.000052$\pm$0.000031 & \textbf{0.000041$\pm$0.000049} & \underline{0.000159$\pm$0.000227} \\
\bottomrule
\end{tabular}
}

\end{table*}

\begin{table}[!t]
\centering
\caption{Final test loss (MSE, mean $\pm$ std over 5 seeds) on the \textbf{Bessel function} and \textbf{Damped SHM} datasets. Best/second-best results are shown in \textbf{bold}/\underline{underlined}.}
\label{tab:bessel}
\resizebox{\textwidth}{!}{
\begin{tabular}{lcccc}
\toprule
Model & Window-size=8 & Window-size=16 & Window-size=32 & Window-size=64 \\
\midrule
\multicolumn{5}{l}{\textbf{\textit{Bessel function}}} \\
\addlinespace[2pt]
QFWP & 0.000126$\pm$0.000100 & 0.003918$\pm$0.001110 & 0.005946$\pm$0.002581 & 0.005269$\pm$0.002289 \\
GQKAN-FWP & \underline{0.000055$\pm$0.000046} & \underline{0.000368$\pm$0.000688} & \underline{0.000673$\pm$0.001169} & \underline{0.000770$\pm$0.001331} \\
G-QKANFWP & 0.000680$\pm$0.001339 & 0.000661$\pm$0.001311 & 0.001995$\pm$0.001637 & 0.002087$\pm$0.001709 \\
GQKAN-QKANFWP & \textbf{0.000025$\pm$0.000027} & \textbf{0.000011$\pm$0.000012} & \textbf{0.000015$\pm$0.000011} & \textbf{0.000021$\pm$0.000024} \\
\cmidrule(lr){1-5}
\multicolumn{5}{l}{\textbf{\textit{Damped SHM}}} \\
\addlinespace[2pt]
QFWP & 0.000233$\pm$0.000145 & 0.003460$\pm$0.004777 & 0.001103$\pm$0.000704 & 0.034814$\pm$0.020245 \\
GQKAN-FWP & \underline{0.000118$\pm$0.000078} & \underline{0.000059$\pm$0.000035} & \underline{0.000097$\pm$0.000052} & \underline{0.000553$\pm$0.000917} \\
G-QKANFWP & 0.002431$\pm$0.001302 & 0.002287$\pm$0.001290 & 0.002151$\pm$0.001073 & 0.002182$\pm$0.001095 \\
GQKAN-QKANFWP & \textbf{0.000089$\pm$0.000071} & \textbf{0.000036$\pm$0.000025} & \textbf{0.000019$\pm$0.000016} & \textbf{0.000043$\pm$0.000034} \\

\bottomrule
\end{tabular}
}
\end{table}

\begin{table}[!t]
\centering
\caption{Final test loss (MSE, mean \(\pm\) std over 5 seeds) on the \textbf{NARMA5} and \textbf{NARMA10} datasets. Best/second-best results are shown in \textbf{bold}/\underline{underlined}.}
\label{tab:narma}
\resizebox{\textwidth}{!}
{
\begin{tabular}{lcccc}
\toprule
Model & Window-size=8 & Window-size=16 & Window-size=32 & Window-size=64 \\
\midrule
\multicolumn{5}{l}{\textbf{\textit{Narma5}}} \\
\addlinespace[2pt]
QFWP & \underline{0.000013$\pm$0.000011} & 0.000035$\pm$0.000017 & 0.000113$\pm$0.000032 & 0.000180$\pm$0.000069 \\
GQKAN-FWP & 0.000021$\pm$0.000008 & 0.000036$\pm$0.000021 & \underline{0.000046$\pm$0.000035} & \underline{0.000079$\pm$0.000069} \\
G-QKANFWP & \textbf{0.000005$\pm$0.000004} & \textbf{0.000012$\pm$0.000006} & \textbf{0.000011$\pm$0.000007} & \textbf{0.000010$\pm$0.000003} \\
GQKAN-QKANFWP & 0.000020$\pm$0.000014 & \underline{0.000028$\pm$0.000018} & 0.000048$\pm$0.000035 & 0.000087$\pm$0.000063 \\
\cmidrule(lr){1-5}
\multicolumn{5}{l}{\textbf{\textit{Narma10}}} \\
\addlinespace[2pt]
QFWP  & \textbf{0.000043$\pm$0.000044} & \textbf{0.000002$\pm$0.000002} & 0.000131$\pm$0.000033 & 0.000131$\pm$0.000050 \\
GQKAN-FWP & 0.000077$\pm$0.000019 & 0.000071$\pm$0.000028 & 0.000074$\pm$0.000011 & 0.000108$\pm$0.000042 \\
G-QKANFWP & \underline{0.000057$\pm$0.000010} & \underline{0.000011$\pm$0.000008} & \textbf{0.000016$\pm$0.000012} & \textbf{0.000020$\pm$0.000015} \\
GQKAN-QKANFWP & 0.000100$\pm$0.000066 & 0.000052$\pm$0.000031 & \underline{0.000055$\pm$0.000029} & \underline{0.000108$\pm$0.000063} \\
\bottomrule
\end{tabular}
}
\end{table}

\begin{table}[!t]
\centering
\caption{Final test loss (MSE, mean \(\pm\) std over 5 seeds) on the \textbf{Delayed Quantum Control} and \textbf{Jaynes-Cummings} datasets. Best/second-best results are shown in \textbf{bold}/\underline{underlined}.}
\label{tab:quantum}
\resizebox{\textwidth}{!}{
\begin{tabular}{lcccc}
\toprule
Model & Window-size=8 & Window-size=16 & Window-size=32 & Window-size=64 \\
\midrule
\multicolumn{5}{l}{\textbf{\textit{Delayed Quantum Control}}} \\
\addlinespace[2pt]
QFWP & 0.000701$\pm$0.001325 & 0.003119$\pm$0.003364 & 0.015240$\pm$0.019944 & 0.016752$\pm$0.026858 \\
GQKAN-FWP & \textbf{0.000033$\pm$0.000041} & \underline{0.000053$\pm$0.000077} & \textbf{0.000053$\pm$0.000038} & \underline{0.000135$\pm$0.000088} \\
G-QKANFWP & 0.000153$\pm$0.000041 & 0.000272$\pm$0.000212 & 0.000142$\pm$0.000026 & \textbf{0.000117$\pm$0.000009} \\
GQKAN-QKANFWP & \underline{0.000048$\pm$0.000070} & \textbf{0.000041$\pm$0.000049} & \underline{0.000089$\pm$0.000054} & 0.000221$\pm$0.000334 \\
\cmidrule(lr){1-5}
\multicolumn{5}{l}{\textbf{\textit{Jaynes-Cummings}}} \\
\addlinespace[2pt]
QFWP & \textbf{0.000092$\pm$0.000117} & 0.010183$\pm$0.019995 & 0.021960$\pm$0.034133 & 0.138087$\pm$0.121635 \\
GQKAN-FWP & \underline{0.000177$\pm$0.000225} & \textbf{0.000070$\pm$0.000074} & \underline{0.000283$\pm$0.000471} & \textbf{0.000037$\pm$0.000044} \\
G-QKANFWP & 0.000677$\pm$0.000046 & 0.000664$\pm$0.000041 & 0.000600$\pm$0.000294 & 0.000666$\pm$0.000357 \\
GQKAN-QKANFWP & 0.000201$\pm$0.000254 & \underline{0.000159$\pm$0.000227} & \textbf{0.000196$\pm$0.000269} & \underline{0.000322$\pm$0.000314} \\
\bottomrule
\end{tabular}
}
\end{table}

\subsection{Real-World Direct Multi-Step Prediction}
\label{subsec:multistep}

While \cref{subsec:time_series} demonstrates our models' ability to capture synthetic dynamics via single-step prediction, practical applications often demand long-horizon, multi-step forecasting in complex, non-stationary environments. To evaluate this, we apply GQKAN-QKANFWP to solar cycle forecasting, a well-known challenge in solar physics \cite{bhowmik2018prediction, petrovay2020solar,pesnell2008predictions}, and conclude this section with an inference test on available real quantum hardware~\cref{subsec:hardware}. We use 3,326 monthly averaged sunspot records spanning 1749--2026 from the World Data Center SILSO\footnote{Data available at \url{http://sidc.be/silso/datafiles}.} \cite{clette2016new}. Following \cite{benson2020SC25}, we frame this as a univariate multi-step forecasting task: a sliding window maps a 528-month input (roughly four solar cycles) to a 132-month forecast horizon (one cycle). The output at the final time step serves as our prediction. To prioritize the accurate prediction of solar cycle maxima \cite{bhowmik2018prediction, petrovay2020solar}, we optimize a peak-aware MSE loss: \begin{equation*}
\mathcal{L} = \frac{1}{B} \sum (\mathbf{y} - \mathbf{\hat{y}})^2 (1 + \alpha \mathbf{y}),
\end{equation*} where $B$ is the batch size and $\alpha=1.0$ scales the penalty for peak values. We also report two additional metrics to quantify peak prediction accuracy in terms of absolute amplitude difference and temporal displacement \cite{bhowmik2018prediction}: peak amplitude error, which measures the absolute difference in sunspot numbers: 
\begin{equation*}
\mathrm{PAE} = \frac{1}{M} \sum_{i=1}^{M} \left| \max(\mathbf{y}^{(i)}) - \max(\mathbf{\hat{y}}^{(i)}) \right|
\end{equation*}
and Peak Timing Error, which measures the temporal displacement in months: 
\begin{equation*}
\mathrm{PTE} = \frac{1}{M} \sum_{i=1}^{M} \left| \operatorname{argmax}(\mathbf{y}^{(i)}) - \operatorname{argmax}(\mathbf{\hat{y}}^{(i)}) \right|,
\end{equation*} 
where $M$ is the number of test sequences, and $\mathbf{y}^{(i)}, \mathbf{\hat{y}}^{(i)}$ are the ground-truth and predicted sequences, respectively.

We benchmark GQKAN-QKANFWP against the WaveNet-LSTM and LSTM baselines 
from \cite{benson2020SC25} (LSTM-L with $H=132$, LSTM-S with $H=64$), as well as the Vanilla RNN and Modified Echo State Network (MESN) baselines from \cite{espuna2023forecasting}. To isolate architectural capacity from training-protocol confounds, we re-run all baselines under a single standardized protocol rather than strictly replicating prior setups. Using the raw monthly data normalized to $[0, 1]$, a 528-month input window, and a 132-month forecast horizon, we chronologically partition the dataset into 80\% training, 10\% validation, and 10\% test sets. All baselines except MESN use batch normalization, 30\% dropout, a batch size of 32, and 100 epochs, evaluated across five random seeds. For these models, learning rates are individually tuned from $\{1, 2, 2.5\} \times 10^{-3}$ on the 
validation split, and the checkpoint with the lowest average validation loss is used for testing. For MESN, we adopt the reservoir configuration of \cite{espuna2023forecasting} unchanged, except for setting the output horizon to 132 instead of 129 to match our forecast window. Its input delay embedding is drawn from the tail of the same 528-month input window fed to the other baselines, so the test split is identical across all models. Because MESN is fit in one pass by closed-form weighted ridge regression, it has no learning rate or epoch budget. Because our protocol evaluates long, raw input sequences rather than the 13-month smoothed data and variable windows used in \cite{espuna2023forecasting}, the baseline metrics reported here reflect performance under stricter conditions. Similarly, our 
WaveNet-LSTM reproduction differs quantitatively from \cite{benson2020SC25} due to our peak-aware loss and our chronological, strictly unseen test split rather than their 5-fold cross-validation scheme.

\begin{figure}[!t]
\centering
\includegraphics{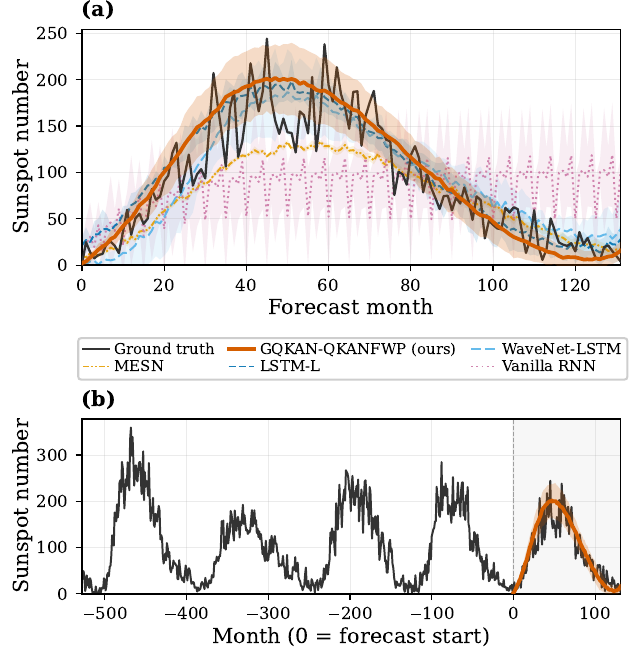}
\caption{\textbf{Forecasting Solar Cycle 23 (test set).} \textbf{(a)} Mean forecasts and shaded $\pm 1\sigma$ bands across 5 random seeds for each model. While the ground truth (black) exhibits substantial 
month-to-month variability, the GQKAN-QKANFWP $\pm 1\sigma$ envelope (orange shading) contains the ground truth throughout the rising, peak, and descending phases of the cycle. Among the baselines, only LSTM-L 
produces a mean prediction that overlaps the GQKAN-QKANFWP envelope, yet uses approximately 7$\times$ more parameters (see \cref{tab:sunspot}). The remaining baselines either systematically under-predict the peak or fail to produce a coherent cycle. 
\textbf{(b)} Full context: four preceding solar cycles followed by the forecast window (shaded).}
\label{fig:sunspot_snapshot}
\end{figure}

\begin{table*}[!t]
\centering
\caption{Solar cycle forecasting performance across models. Baseline architectures follow \cite{benson2020SC25} (LSTM, WaveNet-LSTM) and \cite{espuna2023forecasting} (Vanilla RNN, MESN). Gradient-based baselines use their individually tuned learning rates; MESN is fit by closed-form weighted ridge regression and has no learning rate. Best results in \textbf{bold}, second-best \underline{underlined}. Values are mean $\pm$ std over 5 seeds.}
\label{tab:sunspot}
\resizebox{\textwidth}{!}{
\begin{tabular}{l r c c c c}
\toprule
Model & Params & Selected LR & Scaled MSE $\downarrow$ & PAE $\downarrow$ & PTE $\downarrow$ \\
\midrule
WaveNet-LSTM \cite{benson2020SC25} & 167,196 & $2\!\times\!10^{-3}$ & $0.0205 \pm 0.0056$ & $43.60 \pm 7.26$ & $27.35 \pm 3.01$ \\
LSTM-L \cite{benson2020SC25} & 89,100 & $2.5\!\times\!10^{-3}$ & $\underline{0.0180 \pm 0.0020}$ & $\underline{39.98 \pm 3.82}$ & $\underline{23.34 \pm 1.93}$ \\
LSTM-S & 25,860 & $2.5\!\times\!10^{-3}$ & $0.0194 \pm 0.0017$ & $41.95 \pm 1.32$ & $24.32 \pm 0.74$ \\
MESN \cite{espuna2023forecasting}  & 132,132 & -- & $0.0458 \pm 0.0042$ & $69.14 \pm 4.19$ & $31.81 \pm 2.09$ \\
Vanilla RNN \cite{espuna2023forecasting} & 11,525 & $2\!\times\!10^{-3}$ & $0.0368 \pm 0.0078$ & $54.27 \pm 11.48$ & $39.59 \pm 8.41$ \\
\midrule
GQKAN-QKANFWP (ours) & 12,474 & $2.5\!\times\!10^{-3}$ & $\mathbf{0.0168 \pm 0.0016}$ & $\mathbf{39.59 \pm 2.82}$ & $\mathbf{21.89 \pm 1.57}$ \\
\bottomrule
\end{tabular}
}
\end{table*}

\begin{figure*}[!t]
\centering
\includegraphics[width=\textwidth]{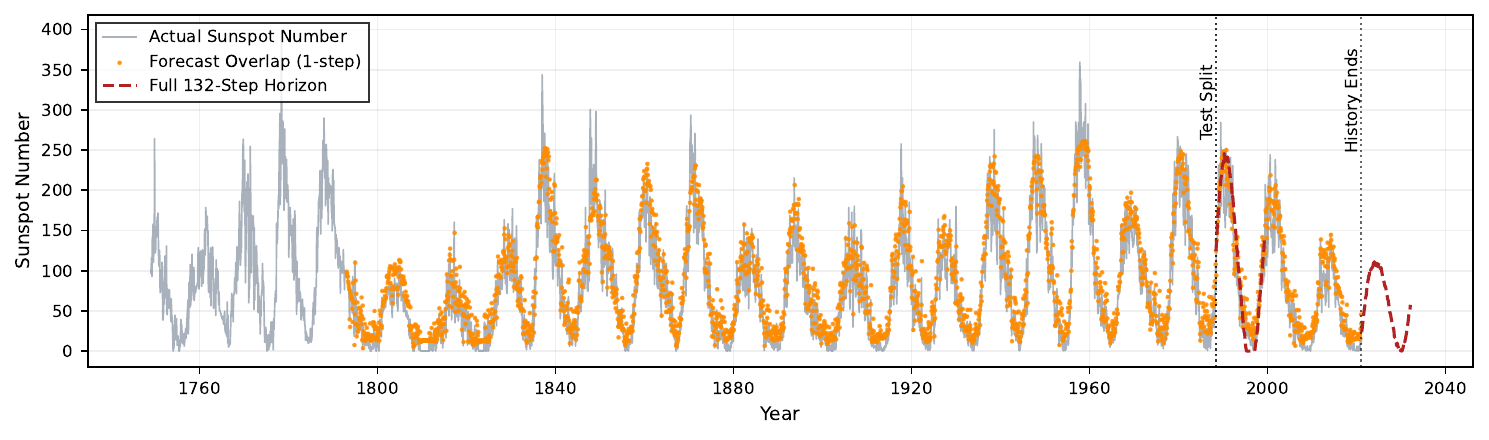}
\caption{\textbf{Solar cycle forecasting results for GQKAN-QKANFWP.} Orange markers represent continuous one-step-ahead forecasts stacking the first step of the 132-step horizon across overlapping sliding windows, while the red curves denote full 132-step predictions on Solar Cycle 22 and the ongoing Solar Cycle 25 generated from a single input window. The ``Test Split'' line marks the beginning of the test set. Additionally, the ``History Ends'' line indicates the separation between the available historical data and the model's future prediction.}
\label{fig:sunspot_forecast}
\end{figure*}

As summarized in \cref{tab:sunspot}, GQKAN-QKANFWP attains the lowest scaled MSE, PAE, and PTE among all evaluated models. This suggests that its advantage reflects a joint improvement in overall fit and peak 
prediction rather than a trade-off between them. This is achieved with only 12.5k parameters, roughly $7$--$13\times$ fewer than the competitive baselines (LSTM-L: 89k; MESN: 132k; WaveNet-LSTM: 167k). The model also exhibits the lowest seed-to-seed variance on scaled MSE ($\pm 0.0016$) and the second-lowest variance on PAE and PTE (after LSTM-S), indicating the reported gains are stable rather than seed-dependent. \cref{fig:sunspot_snapshot} visualizes the per-model performance on Solar Cycle 23 (SC23) as seed-averaged forecasts with $\pm 1\sigma$ bands. GQKAN-QKANFWP's $\pm 1\sigma$ envelope (orange shading) contains the ground truth throughout the rising, peak, and descending phases of the cycle. LSTM-L is the only baseline whose mean prediction overlaps this envelope throughout the cycle but has approximately 7$\times$ more parameters. The remaining baselines either systematically under-predict the solar maximum or fail to form a coherent cycle. \cref{fig:sunspot_forecast} illustrates the overall forecasting behavior: orange markers denote continuous one-step-ahead forecasts obtained by stacking the first value of each 132-month horizon across overlapping sliding windows, while red curves show the full 132-step predictions for SC22 and the ongoing SC25, each generated from a single input window in the test set. The model captures the macroscopic cycle structure on SC22 and produces a stable projection for SC25. Overall, these results suggest that GQKAN-QKANFWP can process long input sequences (528 months) and produce direct multi-step forecasts over a long horizon (132 months) while maintaining low overall MSE and preserving the amplitude and timing of the cycle maxima, a regime in which substantially larger classical recurrent baselines, under the same training protocol, tend to degrade on at least one of these axes.

\begin{table}[!t]
\centering
\caption{Execution of GQKAN-QKANFWP's fast programmer on QPUs. Relative MSE is the MSE with respect to the noiseless simulator output horizon.}
\label{tab:hardware}
\resizebox{\textwidth}{!}{
\begin{tabular}{l r r r r r}
\toprule
Device & Shots $N$ & Scaled MSE $\downarrow$ & PAE $\downarrow$ & PTE $\downarrow$ & Relative MSE $\downarrow$ \\
\midrule
\texttt{Forte-1} & 1024 & 0.00601 & 2.91 & 12.0 & 0.00082 \\
\midrule
\texttt{ibm\_aachen} & 1 & 0.00716 & 1.12 & 12.0 & 0.00759 \\
 & 16 & 0.01065 & 17.45 & 16.0 & 0.04028 \\
 & 64 & 0.00774 & 12.26 & 12.0 & 0.00497 \\
 & 256 & 0.00585 & 0.33 & 12.0 & 0.00175 \\
 & 1024 & 0.00574 & 1.44 & 12.0 & 0.00085 \\
\midrule
Simulator & ---  & 0.00593 & 0.73  & 12.0 & --- \\
\bottomrule
\end{tabular}
}
\end{table}

\subsubsection{Execution on real quantum hardware}
\label{subsec:hardware}
To validate NISQ compatibility, we deploy the fast programmer of a trained GQKAN-QKANFWP onto IonQ's \texttt{Forte-1} and IBM's \texttt{ibm\_aachen} QPUs. The slow programmer and the gated fast-parameter recursion are evaluated classically, while the fast programmer runs on the QPUs, isolating the DARUAN module to quantify noise impact without retraining. The model utilizes 200 single-qubit DARUAN circuits, which execute in parallel. On the 156-qubit \texttt{ibm\_aachen} (Heron r3), we selected 100 qubits via a composite calibration score (readout error $\in [2.6, 9.0]\times10^{-3}$, SX error $\in [0.8, 6.3]\times10^{-4}$), applying Qiskit's \texttt{optimization\_level=1} without further error mitigation. For the 36-qubit \texttt{Forte-1}, we uniformly pinned the first 20 qubits as Braket's device-level aggregates (single-qubit randomized-benchmarking fidelity $\approx 0.9998$, SPAM fidelity $\approx 0.9937$) preclude meaningful per-qubit ranking. We evaluate the full shot sweep $N \in \{1, 16, 64, 256, 1024\}$ on \texttt{ibm\_aachen} to characterize the convergence-with-shots behavior, and report the converged high-shot setting $N=1024$ on \texttt{Forte-1} as an independent cross-platform check. We report scaled MSE, PAE, PTE, and the relative MSE with respect to the noiseless simulator. As shown in \cref{tab:hardware} and \cref{fig:hardware}, GQKAN-QKANFWP's forecasts on both QPUs converge within $\sim 8.5\times10^{-4}$ of relative MSE at $N=1024$ shots---saturating the $\mathcal{O}(1/N) \sim10^{-3}$ statistical floor imposed by shot-noise-limited expectation-value estimation \cite{giovannetti2004quantum,knill2007optimal}.

\begin{figure}[!t]
\centering
\includegraphics{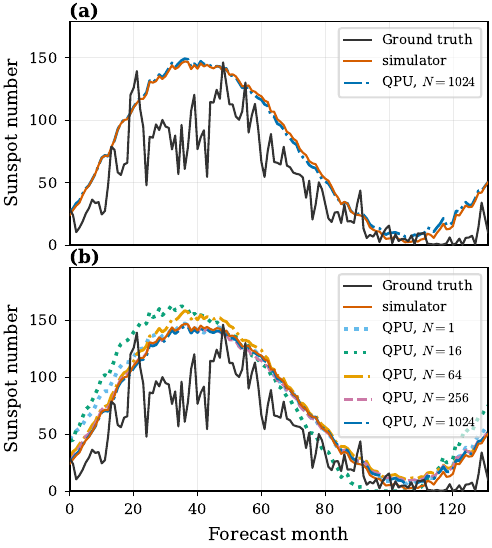}
\caption{\textbf{Forecasting Solar Cycle 24 from GQKAN-QKANFWP's fast programmer executed on QPUs.} (a) \texttt{Forte-1} at $N = 1024$ shots. (b) \texttt{ibm\_aachen} across shot counts $N \in \{1, 16, 64, 256, 1024\}$; forecasts converge to the noiseless simulator as $N$ increases, recovering cycle shape and peak within $\sim10^{-3}$ relative MSE at $N{=}1024$.}
\label{fig:hardware}
\end{figure}

\subsection{Reinforcement learning}
\label{subsec:rl}
Following the benchmark of \cite{chen2024qfwp}, we evaluate RL agents on the MiniGrid-Empty environment \cite{MinigridMiniworld23} trained with asynchronous advantage actor-critic (A3C) \cite{chen2023asynchronous}. Since QFWP has been shown to converge substantially faster than QLSTM \cite{chen2024qfwp,ceschini2026qfwp}, we restrict the comparison to the FWP variants. At each step the agent receives a 147-dimensional observation ($7{\times}7{\times}3$ flattened local viewport) and selects among seven discrete actions, with sparse reward $R = 1 - 0.9\,\frac{\text{steps}}{\text{max steps}}$ on reaching the goal ($\text{max steps} = 4n^2$ for an $n{\times}n$ grid) and zero otherwise. We evaluate $n \in \{5,6,8,16\}$ as shown in~\cref{fig:minigrids}. Each model is trained for 10{,}000 episodes with 80 workers, learning rate $1{\times}10^{-4}$, $\beta_1{=}0.92$, $\beta_2{=}0.999$, rollout length $L{=}5$, and a discount factor $\gamma{=}0.9$. Reported rewards are smoothed by a 100-episode moving average and averaged across workers and seeds.

\begin{figure}[!t]
\centering
\includegraphics[width=\textwidth,trim={0cm 15.1cm 28cm 0cm}, clip]{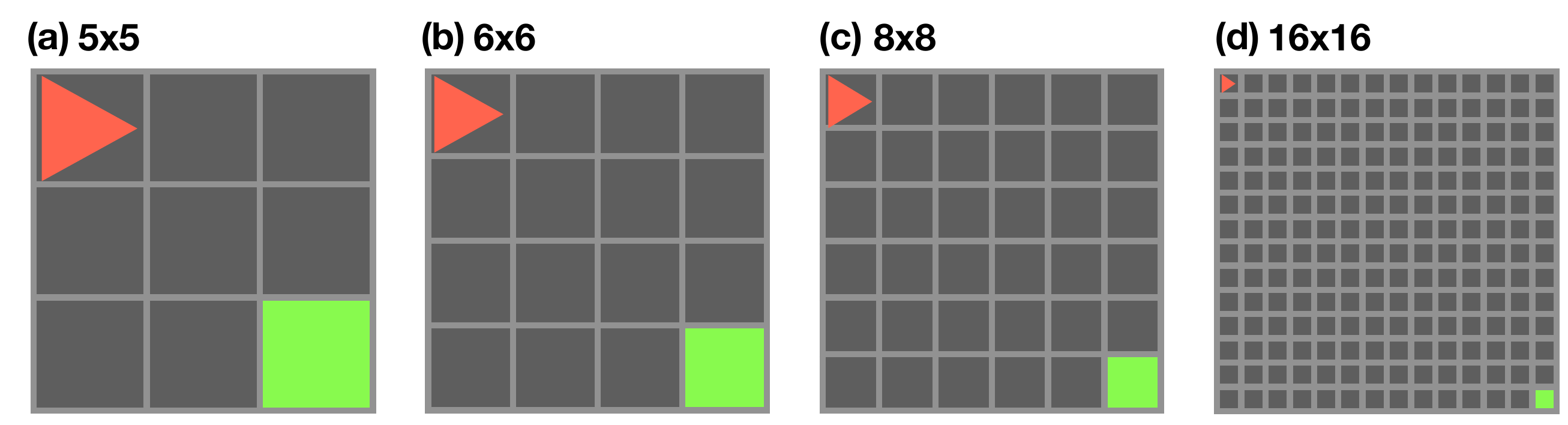}
\caption{\textbf{MiniGrid-Empty environments.} The agents are evaluated across environments of increasing scale: (a) 5$\times$5, (b) 6$\times$6, (c) 8$\times$8, and (d) 16$\times$16.}
\label{fig:minigrids}
\end{figure}

\cref{fig:rl_multi_env}(a) shows that adding the gate generally improves both convergence stability and final performance. Final rewards across all environment sizes are reported in \cref{tab:minigrid_results}, with full learning curves in \cref{fig:rl_multi_env}(b)--(e). Ungated QFWP degrades substantially as the grid grows, while the gated variants remain stable. Among them, G-QKANFWP attains the highest or second-highest final reward in the larger environments ($6{\times}6$, $8{\times}8$, and $16{\times}16$) despite slightly slower early convergence, indicating that the HQKAN fast programmer retains capacity in more complex state spaces. The fully HQKAN-based GQKAN-QKANFWP reaches competitive rewards with substantially fewer parameters: on the $16{\times}16$ task it achieves $0.974\pm0.001$ with only 1{,}114 trainable parameters, versus $0.975\pm0.001$ for G-FWP at 2{,}665 parameters, a $\sim$58\% reduction at essentially matched performance. \cref{fig:rl_multi_env}(e) further shows that GQKAN-FWP converges slightly faster than classical G-FWP on the $16{\times}16$ grid, suggesting a training-efficiency benefit from the quantum-inspired slow programmer even when the fast programmer remains classical.

\begin{figure}[!t]
\centering

\includegraphics[width=\textwidth]{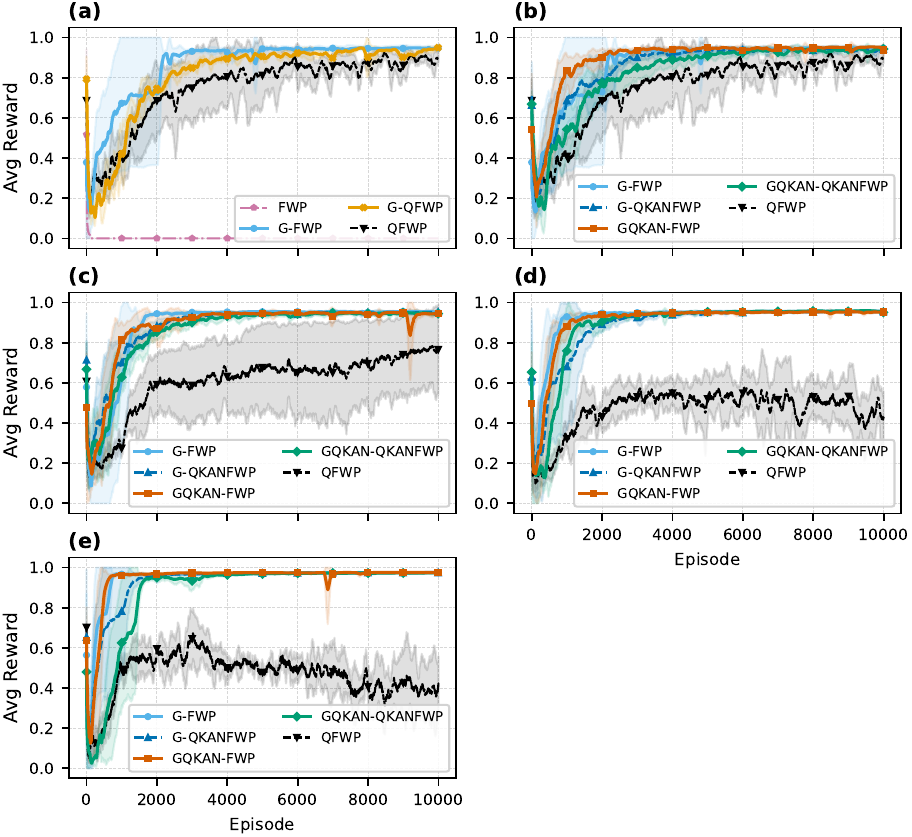}
\caption{\textbf{Model performance on MiniGrid-Empty environments.} The curves show mean episodic reward with shaded regions denoting standard deviation across 5 seeds. \textbf{(a)} Gated-vs-ungated ablation on the $5{\times}5$ grid: gated architectures yield higher stability and asymptotic rewards than their ungated counterparts. \textbf{(b)--(e)} Scaling across $5{\times}5$, $6{\times}6$, $8{\times}8$, and $16{\times}16$ grids for the top-performing variants.}
\label{fig:rl_multi_env}

\end{figure}

\begin{table}[!t]

\centering
\caption{Final reward on MiniGrid-Empty tasks (mean \(\pm\) std over five seeds). Higher is better. Best/second-best results in each column are shown in \textbf{bold}/\underline{underlined}.}
\label{tab:minigrid_results}

{
\resizebox{\textwidth}{!}{
\begin{tabular}{lcccc}
\toprule
Model & 5$\times$5 & 6$\times$6 & 8$\times$8 & 16$\times$16 \\
\midrule
QFWP             & 0.904 $\pm$ 0.027             & 0.765 $\pm$ 0.224             & 0.423 $\pm$ 0.130             & 0.423 $\pm$ 0.135 \\
G-FWP            & \underline{0.948 $\pm$ 0.012} & \underline{0.951 $\pm$ 0.010} & \underline{0.954 $\pm$ 0.005} & \textbf{0.975 $\pm$ 0.001} \\
G-QFWP           & \textbf{0.950 $\pm$ 0.005}    & 0.949 $\pm$ 0.007             & 0.946 $\pm$ 0.007             & 0.969 $\pm$ 0.007 \\
GQKAN-QFWP       & 0.942 $\pm$ 0.007             & 0.936 $\pm$ 0.012             & 0.945 $\pm$ 0.011             & 0.972 $\pm$ 0.000 \\
GQKAN-FWP        & 0.938 $\pm$ 0.029             & 0.945 $\pm$ 0.012             & 0.953 $\pm$ 0.006 & \textbf{0.975 $\pm$ 0.001} \\
G-QKANFWP        & 0.943 $\pm$ 0.015             & \textbf{0.953 $\pm$ 0.005}    & \textbf{0.955 $\pm$ 0.006}    & \underline{0.974 $\pm$ 0.001} \\
GQKAN-QKANFWP    & 0.943 $\pm$ 0.010             & 0.945 $\pm$ 0.007             & 0.951 $\pm$ 0.008             & \underline{0.974 $\pm$ 0.001} \\
\bottomrule
\end{tabular}
}
}

\end{table}

\section{Conclusion}
We presented gated QKAN-FWP, a quantum-inspired sequence learning framework that mitigates the scalability and execution bottlenecks inherent to the NISQ era. By relying exclusively on HQKAN—modules empirically shown capable of scaling to LLMs \cite{jiang2025qkan}—our framework inherits strong scalability while circumventing the costs of multi-qubit entanglement. A scalar-gated fast-weight update further stabilizes parameter evolution, a property we interpret theoretically through adaptive memory kernels, geometric boundedness, and a parallel-scan-compatible recursion. Empirical evaluations highlight the architecture's versatility and parameter efficiency. In time-series prediction, HQKAN-based gated variants exhibited the greatest robustness over extended input windows. On real-world solar cycle forecasting, our 12.5k-parameter GQKAN-QKANFWP achieved lower scaled MSE, PAE, and PTE than a suite of classical recurrent baselines spanning 11.5k to 167k parameters—up to 13× larger than ours—including LSTM-L, WaveNet-LSTM, and MESN. Furthermore, deploying the trained fast programmer on IonQ's \texttt{Forte-1} and IBM's \texttt{ibm\_aachen} recovered forecasting accuracy within $\sim10^{-3}$ relative MSE of the simulator at 1024 shots, confirming the NISQ compatibility of the single-qubit design. In MiniGrid RL task, the framework achieved competitive performance with a 58\% parameter reduction relative to prior baselines.

Although original KAN architectures face optimization challenges in ultra-large-scale scenarios \cite{noorizadegan2026practitionersguidekolmogorovarnoldnetworks,yang2025kolmogorovarnold}, our framework structurally mitigates this burden. Processing sequences autoregressively keeps the HQKAN input dimension independent of sequence length, and operating within HQKAN's reduced-dimensional latent space compresses computational overhead; for tasks with ultra-large input/output dimensions, this overhead can be further managed via structural grouping \cite{yang2025kolmogorovarnold,jiang2025qkan_github}. Pushing these dimensional limits will therefore drive our future work, alongside analyzing optimization dynamics and extending execution on physical quantum hardware beyond inference.

\section*{Acknowledgment}
K.-C. Peng, J.-C. Jiang, Y.-C. Hsu and C.-H. Lin thank the National Center for High-Performance Computing (NCHC), National Institutes of Applied Research (NIAR), Taiwan, for providing computational and storage resources supported by the National Science and Technology Council (NSTC), Taiwan, under Grants No. NSTC 114-2119-M-007-013.
H.-S. Goan acknowledges support from the NSTC, Taiwan, under Grants No. NSTC 113-2112-M-002-022-MY3, No. NSTC 113-2119-M-002-021, No. NSTC 114-2119-M-002-018, No. NSTC 114-2119-M-002-017-MY3, and from the National Taiwan University under Grants No. NTU-CC-115L8937, No. NTU-CC-115L893704 and No. NTU-CC-115L8512. H.-S. Goan is also grateful for the support of the “Center for Advanced Computing and Imaging in Biomedicine (NTU-115L900702)” through the Featured Areas Research Center Program within the framework of the Higher Education Sprout Project by the Ministry of Education (MOE), Taiwan, the support of Taiwan Semiconductor Research Institute (TSRI) through the Joint Developed Project (JDP) and the support from the Physics Division, National Center for Theoretical Sciences, Taiwan. EJK acknowledges financial support from the NSTC of Taiwan under Grant No.~NSTC~114-2112-M-A49-036-MY3. The authors acknowledge the National Taiwan University--IBM Quantum Hub (NTU--IBM Q Hub) and Cloud Computing Center for Quantum Science \& Technology at National Cheng Kung University for providing IBM Q system and Amazon Braket platforms.

\bibliographystyle{alpha}
\bibliography{ref}

\appendix

\section{Parallel Evaluation of the Gated Fast-Weight Recursion}
\label{app:parallel_depth}

This appendix establishes that the gated fast-weight recursion admits efficient parallel evaluation in the forward pass. We first reformulate the update~\cref{eq:gated} as an affine recurrence with a scalar multiplier and show that the induced composition operator on parameter pairs is associative. This structure allows us to express the trajectory $\{W_t\}_{t=1}^{T}$ as a prefix product, which can be evaluated using a work-efficient parallel scan. We then analyze the resulting complexity, showing that the full trajectory can be computed in $O(\log T)$ depth on an unbounded-processor PRAM and in $O(T/p + \log p)$ time on $p$ processors. Finally, we contrast this behavior with the $\Omega(T)$ sequential depth required for general nonlinear recurrences, highlighting the structural source of parallelism in the gated update.

\subsection{Affine Reformulation and Associativity}
\label{app:affine}

Throughout this appendix we work with the gated fast-weight recursion
\begin{equation}
W_{t+1} \;=\; g_t\, W_t \,+\, (1 - g_t)\,\Delta W_t,
\qquad g_t \in [0,1],
\quad t = 1, 2, \dots, T,
\label{eq:app_gated}
\end{equation}
with initial fast parameters $W_1 \in \mathbb{R}^{m\times n}$. The gate $g_t \in [0,1]$ and proposal $\Delta W_t \in \mathbb{R}^{m\times n}$ are produced by the slow programmer from the input $x_t$ alone:
\begin{equation}
\Delta W_t \;=\; S_{\Delta}(x_t),
\qquad
g_t \;=\; \sigma\!\bigl(S_g(x_t)\bigr),
\label{eq:app_slow}
\end{equation}
so that $\{(\Delta W_t, g_t)\}_{t=1}^{T}$ depend only on the input sequence $\{x_t\}_{t=1}^{T}$ and slow-programmer weights, never on previous fast parameters $W_{<t}$. This independence is what makes the entire set of pairs computable in a single parallel pass over time.

The remaining task — and the technical content of this appendix — is to show that the parameter trajectory $\{W_t\}_{t=1}^{T+1}$ \emph{itself} can be resolved in parallel, despite the apparent sequential coupling in~\cref{eq:app_gated}.

Define
\begin{equation}
a_t \;:=\; g_t \in [0,1],
\qquad
b_t \;:=\; (1 - g_t)\,\Delta W_t \in \mathbb{R}^{m\times n}.
\label{eq:app_ab_def}
\end{equation}
Substituting into~\cref{eq:app_gated} yields
\begin{equation}
W_{t+1} \;=\; a_t\, W_t \,+\, b_t,
\label{eq:app_affine}
\end{equation}
which is affine in $W_t$ with a \emph{scalar} multiplier $a_t \in [0,1]$ acting by ordinary scalar multiplication on $W_t \in \mathbb{R}^{m\times n}$.

The scalar nature of $a_t$ is essential. If, instead, the multiplier were a matrix $A_t \in \mathbb{R}^{m\times m}$ acting by left multiplication (as in element-wise gated RNNs), composition would still be associative, but each composed element would be a dense $m\times m$ matrix; the prefix scan would carry $O(m^2)$ payload per node, and gradient composition would require multiplying $T$ dense matrices, reintroducing the conditioning issues we wish to avoid. The single scalar gate $g_t$ keeps both the scan payload dimension and the gradient chain trivial.

Consider the set of \emph{affine pairs} $\mathcal{A} := \mathbb{R} \times \mathbb{R}^{m\times n}$, where the first component is a scalar multiplier and the second is a matrix offset. Define the binary operator $\circ : \mathcal{A} \times \mathcal{A} \to \mathcal{A}$ by
\begin{equation}
(a', b') \circ (a, b) \;:=\; \bigl(a'a,\; a'b + b'\bigr).
\label{eq:app_compose}
\end{equation}
Geometrically, $(a, b)$ represents the affine map $W \mapsto aW + b$, and~\cref{eq:app_compose} expresses the composition of two such maps in the conventional left-to-right order:
\begin{equation}
\bigl[(a',b') \circ (a,b)\bigr](W)
\;=\; a'(aW + b) + b'
\;=\; (a'a)W + (a'b + b').
\label{eq:app_compose_geom}
\end{equation}

The associativity of $\circ$ is a specialization of the standard 
recurrence-to-scan reduction for first-order linear recurrences 
\cite[\S4.1]{blelloch1990prefix} to the case where the multiplier is a scalar in $\mathbb{R}$ and the offset is a matrix in $\mathbb{R}^{m\times n}$. We record the proof in our setting both for completeness and to fix notation for the gradient analysis in \cref{app:gradient}.

\begin{lemma}[Associativity of $\circ$]
\label{lem:assoc}
For any $(a,b),\,(a',b'),\,(a'',b'') \in \mathcal{A}$,
\[
\bigl((a'',b'') \circ (a',b')\bigr) \circ (a,b)
\;=\;
(a'',b'') \circ \bigl((a',b') \circ (a,b)\bigr).
\]
\end{lemma}

\begin{proof}
We expand both sides directly from the definition~\cref{eq:app_compose}.

\emph{Left-hand side.} First compute the inner composition:
\[
(a'',b'') \circ (a',b') \;=\; (a''a',\; a''b' + b'').
\]
Then compose with $(a,b)$:
\begin{align*}
\bigl(a''a',\; a''b' + b''\bigr) \circ (a,b)
&= \bigl((a''a')\,a,\;\; (a''a')\,b + (a''b' + b'')\bigr) \\
&= \bigl(a''a'a,\;\; a''a'b + a''b' + b''\bigr).
\end{align*}

\emph{Right-hand side.} First compute the inner composition:
\[
(a',b') \circ (a,b) \;=\; (a'a,\; a'b + b').
\]
Then compose with $(a'',b'')$:
\begin{align*}
(a'',b'') \circ \bigl(a'a,\; a'b + b'\bigr)
&= \bigl(a''(a'a),\;\; a''(a'b + b') + b''\bigr) \\
&= \bigl(a''a'a,\;\; a''a'b + a''b' + b''\bigr).
\end{align*}

The two expressions are identical, completing the proof.
\end{proof}

Three remarks are in order.

\paragraph{Remark 1 (no constancy or smoothness assumption on the gates).} The proof of \cref{lem:assoc} treats $a, a', a''$ as arbitrary scalars and $b, b', b''$ as arbitrary matrices. Associativity therefore holds pointwise for every triple of pairs, regardless of whether the scalars are constants, time-varying, or input-dependent. In particular, no assumption such as $g_t \equiv g$ or smoothness of $g_t$ in $t$ is required. The gates produced by the slow programmer in~\cref{eq:app_slow} satisfy this trivially.

\paragraph{Remark 2 (identity element).} The pair $(1,\, 0_{m\times n})$ is a two-sided identity for $\circ$:
\[
(1, 0) \circ (a, b) \;=\; (a, b) \;=\; (a, b) \circ (1, 0),
\]
which lets us extend prefix products to indices $\le 0$ when convenient by padding with $(1,0)$.

\paragraph{Remark 3 (non-commutativity).} The operator $\circ$ is \emph{not} commutative: in general,
\[
(a',b') \circ (a,b) \;=\; (a'a,\, a'b + b')
\;\ne\; (aa',\, ab' + b) \;=\; (a,b) \circ (a',b'),
\]
since the offset components differ. This is consistent with the order-sensitive nature of sequence modeling: the proposal at time $k$ should influence $W_{t+1}$ via gates $g_{k+1},\dots,g_t$, not via gates $g_1,\dots,g_{k-1}$. Associativity, but not commutativity, is what enables parallel reassociation while preserving sequence order.

\subsection{Trajectory as a Prefix Product}
\label{app:prefix_form}

We now show that one step of the recursion~\cref{eq:app_affine} is exactly one application of $\circ$, and consequently that the entire trajectory is a prefix product.

\begin{proposition}[Prefix-product form]
\label{prop:prefix}
Let $W_1 \in \mathbb{R}^{m\times n}$, $a_t \in \mathbb{R}$, and $b_t \in \mathbb{R}^{m\times n}$ for $t = 1, \dots, T$, and define $W_{t+1} = a_t W_t + b_t$ as in~\cref{eq:app_affine}. Define the running prefix product
\begin{equation}
P_t \;:=\; (a_t, b_t) \circ (a_{t-1}, b_{t-1}) \circ \cdots \circ (a_1, b_1) \circ (1, W_1),
\qquad t = 1, \dots, T.
\label{eq:app_prefix}
\end{equation}
Then for every $t$,
\begin{equation}
P_t \;=\; \Bigl(\alpha_t,\; W_{t+1}\Bigr),
\qquad
\alpha_t \;=\; \prod_{s=1}^{t} a_s,
\label{eq:app_prefix_form}
\end{equation}
i.e., the second component of the $t$-th prefix product is exactly the fast-parameter state at time $t+1$.
\end{proposition}

\begin{proof}
We proceed by induction on $t$.

\emph{Base case ($t = 1$).} By direct computation,
\[
P_1 \;=\; (a_1, b_1) \circ (1, W_1) \;=\; \bigl(a_1 \cdot 1,\;\; a_1 W_1 + b_1\bigr) \;=\; (a_1,\; W_2).
\]
Hence $P_1 = (\alpha_1, W_2)$ with $\alpha_1 = a_1$, as claimed.

\emph{Inductive step.} Assume $P_t = (\alpha_t, W_{t+1})$ with $\alpha_t = \prod_{s=1}^{t} a_s$. Then
\begin{align*}
P_{t+1} &\;=\; (a_{t+1}, b_{t+1}) \circ P_t \\
        &\;=\; (a_{t+1}, b_{t+1}) \circ (\alpha_t, W_{t+1}) \\
        &\;=\; \bigl(a_{t+1}\alpha_t,\;\; a_{t+1} W_{t+1} + b_{t+1}\bigr) \\
        &\;=\; \bigl(\alpha_{t+1},\;\; W_{t+2}\bigr),
\end{align*}
where the last equality uses $\alpha_{t+1} = a_{t+1}\alpha_t$ and the recursion $W_{t+2} = a_{t+1} W_{t+1} + b_{t+1}$. Hence $P_{t+1}$ also has the claimed form, completing the induction.
\end{proof}

\paragraph{Worked example for $T = 3$.}
To make \cref{prop:prefix} concrete, we unroll the prefix product for $T = 3$. Composing left-to-right starting from $(1, W_1)$:
\begin{align*}
P_1 \;=\; (a_1, b_1) \circ (1, W_1)
&\;=\; \bigl(a_1,\;\; a_1 W_1 + b_1\bigr)
\;=\; (a_1,\, W_2),\\[2pt]
P_2 \;=\; (a_2, b_2) \circ P_1
&\;=\; (a_2, b_2) \circ (a_1, W_2) \\
&\;=\; \bigl(a_2 a_1,\;\; a_2 W_2 + b_2\bigr)
\;=\; (a_2 a_1,\, W_3),\\[2pt]
P_3 \;=\; (a_3, b_3) \circ P_2
&\;=\; (a_3, b_3) \circ (a_2 a_1, W_3) \\
&\;=\; \bigl(a_3 a_2 a_1,\;\; a_3 W_3 + b_3\bigr)
\;=\; (a_3 a_2 a_1,\, W_4).
\end{align*}
Substituting $a_t = g_t$ and $b_t = (1-g_t)\Delta W_t$ recovers, in the second component of $P_3$,
\[
W_4 \;=\; g_3 g_2 g_1\, W_1 + g_3 g_2 (1-g_1)\Delta W_1 + g_3 (1-g_2)\Delta W_2 + (1-g_3)\Delta W_3,
\]
which matches term-by-term the unrolled form~\cref{eq:gated_unroll} of the main text. This verifies the equivalence between the recursive and prefix-product views.

\paragraph{Tree-shaped reassociation.}
Crucially, by associativity (\cref{lem:assoc}), the same prefix product $P_3$ can be evaluated by any parenthesization of the operands, including the balanced tree
\[
P_3 \;=\; \bigl((a_3, b_3) \circ (a_2, b_2)\bigr) \circ \bigl((a_1, b_1) \circ (1, W_1)\bigr),
\]
in which the two inner compositions are independent and can be performed concurrently. This is the foundation of the parallel scan in \cref{app:scan_depth}.

\subsection{Parallel Evaluation via Associative Scan}
\label{app:scan_depth}

Any associative binary operator admits a work-efficient parallel prefix scan~\cite{blelloch1990prefix}. We briefly recall the standard up-sweep / down-sweep algorithm for completeness, then quantify its depth on the operator $\circ$.

\paragraph{Up-sweep.} Place the inputs $\{(a_k, b_k)\}_{k=1}^{T}$ at the leaves of a balanced binary tree of height $h := \lceil \log_2 T \rceil$. At each internal node, combine the two children under $\circ$. Because all combinations at a given level are independent, each level is performed in parallel on an unbounded-processor PRAM in $O(1)$ time, giving $O(\log T)$ time in total. After the up-sweep, each internal node $v$ holds the reduction of all leaves in its subtree, and the root holds the full reduction $P_T$.

\paragraph{Down-sweep.} A second pass propagates partial prefixes back down the tree. Initialize the root with the identity $(1, 0)$. At each internal node, the left child receives the value passed down from the parent, and the right child receives that value composed (via $\circ$) with the up-sweep value of the left child. After the down-sweep reaches the leaves, leaf $k$ holds the exclusive prefix $(a_{k-1}, b_{k-1}) \circ \cdots \circ (a_1, b_1)$; composing with the leaf's own value yields the inclusive prefix $P_k = (a_k, b_k) \circ \cdots \circ (a_1, b_1) \circ (1, W_1)$, whose second component is $W_{k+1}$ by \cref{prop:prefix}. The down-sweep also takes $O(\log T)$ time on unbounded processors.

\paragraph{Depth and work.} The total depth is $O(\log T)$ and the total work (number of $\circ$ operations) is $O(T)$, matching the sequential cost up to a constant factor. Each $\circ$ operation, by~\cref{eq:app_compose}, requires one scalar multiplication, one scalar-times-matrix scaling, and one matrix addition, i.e., $O(mn)$ scalar operations for $W \in \mathbb{R}^{m\times n}$. The parallel scan therefore performs $O(Tmn)$ scalar operations in total, identical (up to constants) to the $T$ scalar-matrix-add steps of the sequential recursion, but with $O(\log T)$ rather than $O(T)$ depth.

On a realistic machine with $p \ll T$ processors, the standard three-phase implementation of an associative scan~\cite{blelloch1990prefix} achieves
\begin{equation}
T_{\text{scan}}(T, p) \;=\; O\!\left(\frac{T}{p} + \log p\right),
\label{eq:app_finite_p}
\end{equation}
with total work $O(T)$. The three phases are:

\begin{enumerate}
\item \textbf{Local reduction.} Partition the $T$ inputs into $p$ contiguous blocks of size $\lceil T/p \rceil$. Each processor sequentially reduces its block under $\circ$ in $O(T/p)$ time, producing $p$ block reductions.
\item \textbf{Tree-based scan over block reductions.} Apply the up-sweep / down-sweep scan of \cref{app:scan_depth} to the $p$ block reductions, yielding the prefix of block reductions in $O(\log p)$ depth.
\item \textbf{Local propagation.} Each processor takes its received prefix and sequentially propagates it across its block, completing the per-leaf prefixes in $O(T/p)$ time.
\end{enumerate}

Summing the three phases gives~\cref{eq:app_finite_p}. When $p = \Theta(T)$, the $T/p$ term is $O(1)$ and~\cref{eq:app_finite_p} reduces to the unbounded-processor depth $O(\log T)$. When $p \ll T$, the $T/p$ term dominates and the parallel speedup over sequential evaluation is linear in $p$.

Martin and Cundy~\cite{martin2018parallelizing} apply this primitive to feature-space linear recurrences of the form $h_t = \lambda_t \odot h_{t-1} + x_t$ inside neural-network cells (SRU, QRNN, GILR-LSTM), demonstrating practical speedups of up to $9\times$ over serial linear-RNN evaluation on modern GPUs. Our analysis above establishes that the same primitive applies to the parameter-space trajectory $\{W_t\}$ of a gated fast-weight programmer, complementing 
prior work on parallel scans for hidden-state recurrences~\cite{martin2018parallelizing}. Crucially, in our setting the scan is performed once over the parameter sequence and produces the entire trajectory $\{W_t\}_{t=1}^{T}$ for a sequence of length $T$, in contrast to nonlinear recurrent architectures, which must serialize over $t$ in both forward and backward passes.

\subsection{Sequential Depth Lower Bound for Nonlinear Recurrences}
\label{app:nonlinear_contrast}

For a recurrence of the general form
\begin{equation}
h_{t+1} \;=\; f(h_t, x_t),
\label{eq:app_nonlinear}
\end{equation}
with $f$ nonlinear in $h_t$, no analogous reassociation is available. Specifically, the two-step composition
\[
h_{t+2} \;=\; f\bigl(f(h_t, x_t),\; x_{t+1}\bigr)
\]
is not, in general, expressible as a single application of an operator that depends only on $\{x_t, x_{t+1}\}$ and a fixed-size summary of the past. As a consequence, computing $h_T$ from $h_0$ requires $\Omega(T)$ sequential applications of $f$, regardless of available parallelism, and no associative-scan acceleration is possible~\cite{martin2018parallelizing}.

This contrast is precisely what differentiates the gated fast-weight framework from recurrent integrations of HQKAN such as QKAN-LSTM~\cite{hsu2025qkan}, in which an HQKAN block sits inside an LSTM cell and depends on the recurrent hidden state $h_{t-1}$. Such architectures inherit the $\Omega(T)$ sequential depth of LSTMs and require BPTT to traverse a chain of $T$ dense hidden-state Jacobians. The gated fast-weight framework decouples parameter-space evolution from any hidden-state loop: each $\Delta W_k$ is computed from $x_k$ alone, and the trajectory $\{W_t\}$ is resolved by an associative scan in $O(\log T)$ depth.

\paragraph{Summary.}
The gated fast-weight recursion admits a parallel evaluation strategy because its affine structure induces an associative composition law. This allows the parameter trajectory to be computed via a prefix scan in logarithmic depth, in contrast to the inherently sequential evaluation of general nonlinear recurrences. This parallel structure is a direct consequence of the scalar gating mechanism and will also play a central role in simplifying gradient composition, as we show next.

\section{Gradient Composition in the Gated Fast-Weight Recursion}
\label{app:gradient}

We now turn to the backward pass and show that the same affine structure underlying the parallel scan also governs gradient composition. In particular, we derive the sensitivity of $W_{t+1}$ to an earlier proposal $\Delta W_k$ and show that the resulting Jacobian reduces to a single scalar multiplier. This leads to a fundamental simplification: instead of propagating gradients through a chain of dense Jacobians as in general recurrent architectures, temporal dependencies are mediated entirely by products of scalar gates. As a result, gradient propagation along the time axis is both shallower and better conditioned.

\subsection{Unrolled form revisited}
\label{app:unroll_revisit}

Recursively expanding~\cref{eq:app_gated} gives the unrolled form (also stated as~\cref{eq:gated_unroll} in the main text):
\begin{equation}
W_{t+1}
\;=\; \left(\prod_{s=1}^{t} g_s\right) W_1
\;+\; \sum_{k=1}^{t} (1 - g_k)\!\left(\prod_{s=k+1}^{t} g_s\right)\!\Delta W_k,
\label{eq:app_unroll}
\end{equation}
with the convention $\prod_{s=k+1}^{t} g_s = 1$ when $k = t$. Define the scalar memory coefficients
\begin{equation}
\beta_{0,t} \;:=\; \prod_{s=1}^{t} g_s,
\qquad
\beta_{k,t} \;:=\; (1 - g_k)\!\prod_{s=k+1}^{t} g_s,
\quad k = 1, \dots, t.
\label{eq:app_betas}
\end{equation}
By induction one verifies $\beta_{0,t} + \sum_{k=1}^{t} \beta_{k,t} = 1$ and $\beta_{k,t} \in [0,1]$ for all $k$, so $W_{t+1}$ lies in the convex hull of $\{W_1, \Delta W_1, \dots, \Delta W_t\}$, recovering the geometric boundedness property of \cref{subsec:gated_theory}.

\subsection{Sensitivity to a single proposal}
\label{app:single_sensitivity}

We compute $\partial W_{t+1} / \partial \Delta W_k$ entrywise. Fix $k \in \{1, \dots, t\}$, and let $W_{t+1}^{(ij)}$ and $\Delta W_k^{(i'j')}$ denote arbitrary entries of $W_{t+1}$ and $\Delta W_k$, respectively. From~\cref{eq:app_unroll}, the entry $W_{t+1}^{(ij)}$ depends on $\Delta W_k$ only through the term $\beta_{k,t}\Delta W_k$, and within that term it depends on $\Delta W_k^{(i'j')}$ only when $(i',j') = (i,j)$, with derivative $\beta_{k,t}$. Hence
\begin{equation}
\frac{\partial W_{t+1}^{(ij)}}{\partial \Delta W_k^{(i'j')}}
\;=\; \beta_{k,t}\, \delta_{ii'}\delta_{jj'},
\label{eq:app_entry_jacobian}
\end{equation}
where $\delta$ is the Kronecker delta. Vectorizing both sides, the Jacobian is
\begin{equation}
\frac{\partial \mathrm{vec}(W_{t+1})}{\partial \mathrm{vec}(\Delta W_k)}
\;=\; \beta_{k,t}\, I_{mn},
\label{eq:app_scalar_jacobian}
\end{equation}
a scalar multiple of the $mn \times mn$ identity. The dense Jacobian collapses to a \emph{single scalar} multiplier $\beta_{k,t}$ propagated independently along each parameter coordinate.

\paragraph{Worked example: $t=3$, $k=1$.}
For $T = 3$ and $k = 1$, \cref{eq:app_unroll} reads
\[
W_4 \;=\; \beta_{0,3} W_1 + \beta_{1,3}\Delta W_1 + \beta_{2,3}\Delta W_2 + \beta_{3,3}\Delta W_3,
\]
with
\[
\beta_{0,3} = g_1 g_2 g_3,
\qquad
\beta_{1,3} = (1-g_1) g_2 g_3,
\qquad
\beta_{2,3} = (1-g_2) g_3,
\qquad
\beta_{3,3} = (1-g_3).
\]
Therefore
\[
\frac{\partial \mathrm{vec}(W_4)}{\partial \mathrm{vec}(\Delta W_1)}
\;=\; (1 - g_1) g_2 g_3 \, I_{mn},
\]
which is bounded by $1$ in absolute value because $g_s \in [0,1]$ for all $s$ and $1 - g_1 \in [0,1]$. One verifies directly that $\beta_{0,3} + \beta_{1,3} + \beta_{2,3} + \beta_{3,3} = 1$, recovering the convex-combination structure.

\subsection{Backpropagation through the fast parameters}
\label{app:bptt_factorization}

Let $\mathcal{L}$ denote a downstream loss whose dependence on $\Delta W_k$ flows through $W_{t+1}$ for various $t \ge k$ (typically through outputs $y_t = F(x_t; W_t)$ for $t > k$). The chain rule gives
\begin{equation}
\frac{\partial \mathcal{L}}{\partial \mathrm{vec}(\Delta W_k)}
\;=\; \sum_{t \ge k} \frac{\partial \mathcal{L}}{\partial \mathrm{vec}(W_{t+1})}\, \frac{\partial \mathrm{vec}(W_{t+1})}{\partial \mathrm{vec}(\Delta W_k)}
\;=\; \sum_{t \ge k} \beta_{k,t}\, \frac{\partial \mathcal{L}}{\partial \mathrm{vec}(W_{t+1})}.
\label{eq:app_chain_rule}
\end{equation}
The gradient flowing from time $t+1$ back to step $k$ is therefore the upstream gradient $\partial \mathcal{L}/\partial \mathrm{vec}(W_{t+1})$ \emph{scalar-rescaled} by $\beta_{k,t}$, with no matrix multiplication along the temporal direction.

Since the slow programmer produces $\Delta W_k$ from $x_k$ alone via an independent forward pass (cf.~\cref{eq:app_slow}), the gradient with respect to slow-programmer weights $\theta_S$ further factors as
\[
\frac{\partial \mathcal{L}}{\partial \theta_S}
\;=\; \sum_{k=1}^{T} \frac{\partial \mathcal{L}}{\partial \mathrm{vec}(\Delta W_k)}\, \frac{\partial \mathrm{vec}(\Delta W_k)}{\partial \theta_S},
\]
where each term $\partial \mathrm{vec}(\Delta W_k)/\partial \theta_S$ traverses only the depth of one slow-programmer forward pass, not a chain of $T$ recurrent steps. The gradient depth across time is therefore controlled entirely by the scalar product $\beta_{k,t}$, while the per-step gradient depth is the (constant) depth of one slow-programmer evaluation.

\subsection{Bounded, non-explosive gradient magnitudes}
\label{app:gradient_bounds}

The scalar coefficients $\beta_{k,t}$ are products of factors in $[0,1]$:
\begin{equation}
0 \;\le\; \beta_{k,t} \;=\; (1-g_k)\!\prod_{s=k+1}^{t} g_s \;\le\; 1.
\label{eq:app_beta_bound}
\end{equation}
Two consequences follow.

\paragraph{No explosion.} The norm of the temporal Jacobian is bounded above by $1$ for every $(k, t)$:
\[
\Bigl\|\,\partial \mathrm{vec}(W_{t+1})/\partial \mathrm{vec}(\Delta W_k)\,\Bigr\|_2
\;=\; \beta_{k,t} \;\le\; 1.
\]
Repeated composition of these factors across time can only contract gradients; it cannot amplify them. This rules out exploding gradients along the time axis by construction, without recourse to gradient clipping or spectral regularization.

\paragraph{Possible vanishing.} If $g_s$ is consistently small for $s \in \{k+1, \dots, t\}$, the product $\prod_{s=k+1}^{t} g_s$ shrinks geometrically and gradients to early $\Delta W_k$ may vanish. This is the standard adaptive-memory trade-off: gates near $1$ retain long-range gradient flow, while gates near $0$ implement aggressive forgetting. Crucially, the gates $g_s$ are produced by the slow programmer from $x_s$ alone, so their values are learned per input rather than fixed, and the model can in principle learn to retain long-range dependencies where the data warrant it.

\subsection{Comparison with dense recurrent Jacobians}
\label{app:rnn_comparison}

For a general nonlinear recurrence~\cref{eq:app_nonlinear}, the analogous sensitivity is a product of dense cell-Jacobians,
\begin{equation}
\frac{\partial h_{t+1}}{\partial h_k}
\;=\; \prod_{s=k+1}^{t} J_s,
\qquad
J_s \;:=\; \frac{\partial f(h_s, x_s)}{\partial h_s} \in \mathbb{R}^{d\times d},
\label{eq:app_dense_jacobian}
\end{equation}
where $d$ is the hidden-state dimension. The product of such matrices is well known to be the source of both exploding and vanishing gradients~\cite{martin2018parallelizing}: its spectral norm is bounded only by $\prod_s \|J_s\|_2$, which grows or shrinks geometrically with $t-k$ and depends on every entry of every intermediate Jacobian. Backpropagation through time therefore requires (i) storing all $T$ activations to recompute or query the $J_s$, and (ii) sequentially multiplying $T$ dense $d\times d$ matrices on the backward pass, costing $O(T d^3)$ work and $\Omega(T)$ depth.

The gated fast-weight framework replaces the dense product $\prod_s J_s$ with the scalar product $\beta_{k,t}$. The asymptotic comparison is summarized in \cref{tab:gradient_comparison}.

\begin{table}[!t]
\centering
\renewcommand{\arraystretch}{1.15}
\begin{tabular}{lcc}
\toprule
& Gated fast-weight (this work) & General nonlinear recurrence \\
\midrule
Temporal Jacobian payload & scalar in $[0,1]$ & dense $d \times d$ matrix \\
Norm bound across $t-k$ steps & $\le 1$ (always) & unbounded \\
Backward-pass depth across time & $O(\log T)$ via scan & $\Omega(T)$ \\
Backward-pass work along time & $O(T\,mn)$ & $O(T\,d^3)$ \\
Susceptibility to explosion & none & yes \\
Susceptibility to vanishing & yes (gate-modulated) & yes \\
\bottomrule
\end{tabular}
\caption{Comparison of temporal gradient composition between the gated fast-weight framework and a general nonlinear recurrence on hidden state $h_t \in \mathbb{R}^d$ (vs.\ fast parameters $W_t \in \mathbb{R}^{m\times n}$).}
\label{tab:gradient_comparison}
\end{table}

\paragraph{Summary.}
Gradient propagation through the gated fast-weight recursion reduces to the composition of scalar coefficients in $[0,1]$, rather than products of dense Jacobians. This yields two key advantages: (i) bounded, non-explosive gradient magnitudes by construction, and (ii) reduced effective depth of temporal gradient paths, which can be evaluated via parallel reductions. Compared to general nonlinear recurrences, this structure leads to both improved conditioning and lower computational complexity in the backward pass.

\section{Convergence Analysis on Time-Series Benchmarks}
\label{app:convergence}

This section provides a detailed, close-up view of the learning behavior of GQKAN-QKANFWP and the standard QFWP in time series benchmark task introduced in \cref{subsec:time_series}. While the main text reports aggregate metrics and qualitative comparisons up to 50 training epochs, here we extend the analysis to 100 epochs and visualize the full convergence trajectories at the most demanding window-size setting $N=64$. This extended view allows us to distinguish between optimization speed and representational limits, and to assess whether performance gaps persist or close with additional training. For each task we display the model predictions at four representative training epochs, with solid lines denoting the seed-averaged mean across five independent random seed initializations and shaded bands indicating the corresponding $\pm 1\sigma$ envelope. This visualization complements the aggregate metrics reported in Tables~\ref{tab:bessel}--\ref{tab:quantum} by exposing temporal qualitative behavior—amplitude tracking, phase alignment, convergence speed, and seed-to-seed stability—that scalar MSE values alone cannot fully capture. Across all six tasks, two recurring trends emerge. First, GQKAN-QKANFWP attains a near-perfect overlap with the ground truth substantially earlier in training than QFWP, and in the smooth-dynamics and quantum-dynamics tasks this overlap is already reached by epoch~15. Second, the GQKAN-QKANFWP $\pm 1\sigma$ envelope remains tight from early epochs onward and barely broadens in the test region, whereas the QFWP envelope is consistently wider and tends to expand past the train/test boundary, indicating both higher seed-to-seed variability and weaker out-of-sample generalization at long input windows. Crucially, the extended training to epoch~100 shows that these gaps do not close with additional training.

\paragraph{Damped SHM.}
\cref{fig:shm} illustrates the forecasting trajectories on the Damped SHM dataset. The target is a smooth, weakly-damped oscillation with a slowly-decaying amplitude envelope and a mildly amplitude-dependent period, jointly testing amplitude tracking, the retention of a slow envelope, and sensitivity to nonlinearity. The GQKAN-QKANFWP produces predictions that are visually indistinguishable from the ground truth from epoch~15 onward, with a $\pm 1\sigma$ band so narrow it remains within the line thickness of the mean curve across both the training and test regions. In contrast, the QFWP baseline systematically under-predicts the oscillation amplitude at every displayed epoch and fails to preserve the damping envelope; its mean prediction stabilizes as a low-amplitude oscillation whose phase progressively drifts relative to the ground truth, and its variance band visibly broadens in the test region. The contrast does not narrow as training progresses: even at epoch~100 the QFWP retains a clear amplitude deficit, indicating that this is a representational limit rather than an optimization gap. The qualitative behavior is consistent with the roughly three-orders-of-magnitude reduction in test MSE at epoch~50 achieved by GQKAN-QKANFWP at $N=64$ on this dataset as shown in \cref{subsec:time_series}.

\paragraph{Bessel function.}
\cref{fig:bessel} reports the learning trajectories on the second-order Bessel function of the first kind, $J_2(x)$, whose envelope decays as a power law ($\sim x^{-1/2}$) rather than exponentially and whose local period drifts mildly with $x$, in contrast to the strict periodicity of Damped SHM. At epoch~15 the GQKAN-QKANFWP already tracks both the period and the slowly-decaying amplitude envelope, with a tight variance band along the entire window. The QFWP captures the dominant frequency but consistently underestimates the early-time amplitude and exhibits a small but persistent phase offset that accumulates over later cycles. From epoch~30 onward GQKAN-QKANFWP refines its amplitude estimate so as to overlay the ground truth, whereas the QFWP mean curve plateaus at a reduced amplitude and continues to drift in phase, particularly past the train/test split. The seed-averaged shaded bands further reveal that the QFWP exhibits noticeably greater run-to-run variability throughout training, while the GQKAN-QKANFWP envelope remains essentially invisible at the displayed scale. These observations align with two orders of magnitude reduction in test MSE at epoch~50 reported quantitatively, and together suggest that a single fixed-depth additive update rule is insufficient to track multi-scale oscillatory structure at long input windows.

\paragraph{NARMA5.}
NARMA5 (\cref{fig:narma5}) is a nonlinear autoregressive sequence of order $n_0=5$ whose target contains sharp, irregular peaks driven by the order-$5$ autoregressive feedback in the recurrence. The qualitative behavior at $N=64$ is informative for two reasons. First, both models predict a near-constant trajectory at epoch~15, reflecting the difficulty of identifying the underlying nonlinear dependence solely from a 64-step input window. Second, the two models diverge sharply thereafter: the GQKAN-QKANFWP begins to recover the peak structure around epoch~30 and progressively sharpens its peaks through epoch~100, producing a mean curve that approximately tracks the ground-truth maxima and minima with a moderate but tightening variance band. The QFWP, by contrast, remains essentially flat across all four displayed epochs and exhibits a wide, weakly-informative variance band that barely contracts during training. This pattern is consistent with our broader observation in \cref{subsec:time_series} that QFWP undergoes substantial degradation at long input windows on the NARMA family, whereas the gated HQKAN-based variants maintain stable predictive behavior.

\paragraph{NARMA10.}
NARMA10 (\cref{fig:narma10}) doubles the autoregressive order to $n_0=10$ and therefore amplifies the difficulty of capturing the autoregressive structure. The qualitative behavior mirrors that of NARMA5 but with a more pronounced gap. By epoch~50 the GQKAN-QKANFWP resolves the larger peaks of the target, and by epoch~100 it tracks both the major and minor variations with a tight variance envelope. The QFWP captures only a smoothed approximation of the trend, missing most of the peak-to-valley structure, and its variance band remains broadest in the early epochs and only modestly tightens over training. The persistently high run-to-run variability of the QFWP at this longer-memory setting indicates that the additive update rule has difficulty stabilizing across seeds, whereas the seed-robustness of GQKAN-QKANFWP is consistent with the geometric boundedness property of the gated update derived in \cref{subsec:gated_theory} the fast parameters are constrained to the convex hull of the historical proposals, which prevents the unbounded additive accumulation that destabilizes QFWP at long $N$.

\paragraph{Delayed Quantum Control.}
The DQC task (\cref{fig:DQC}) consists of localized pulses with a decaying envelope, and therefore requires the model to retain temporal structure across multiple delay intervals. The GQKAN-QKANFWP reproduces both the pulse shape and the decaying amplitude envelope from epoch~15 onward, with a variance band so narrow it remains within the mean curve. The QFWP qualitatively tracks the dominant pulse structure in the training region but visibly degrades past the train/test split: the mean curve undershoots the pulse peaks, and the variance band broadens. Although both models capture the gross periodicity of the signal, the quantitative gap between them spans roughly two orders of magnitude in test MSE at epoch~50, and this gap is most evident in the unseen test region. This supports the claim that the gated update rule preserves long-range temporal structure that the additive QFWP loses at long input windows—precisely the regime in which non-Markovian feedback through the bound-state-in-the-continuum mechanism makes accurate forecasting most demanding.

\paragraph{Jaynes--Cummings dynamics.}
The JC dataset (\cref{fig:JC}) combines rapid cavity-qubit oscillations with dissipative photon-loss decay, producing the highest-frequency target in our benchmark suite. Already at epoch~15 the GQKAN-QKANFWP overlays the ground truth across the full sequence, capturing both the carrier oscillation and the slowly-decaying amplitude envelope, with a $\pm 1\sigma$ band so narrow it remains within the line thickness of the mean curve. Subsequent epochs (30, 50, 100) preserve this alignment with no visible drift, indicating that the model converges early and stably on this task. The QFWP, in contrast, exhibits a persistent amplitude deficit at every displayed epoch—its mean curve resolves the carrier frequency but underestimates its amplitude by roughly half, and the variance band visibly broadens past the train/test split, with extrapolation errors growing toward the end of the sequence. Although QFWP slowly recovers some amplitude through training, even at epoch~100 it fails to match the ground-truth envelope, confirming that this is a representational rather than an optimization gap. The visual gap between the two models is the most dramatic in our benchmark suite, mirroring the largest quantitative gap as well: the QFWP test MSE on this dataset at $N=64$ exceeds the corresponding GQKAN-QKANFWP value by roughly three orders of magnitude. We interpret this gap as a joint consequence of (i)~the spectral expressivity of the HQKAN-based fast programmer, which provides a rich Fourier basis well-suited to high-frequency dynamics, and (ii)~the gated update rule, which prevents the additive accumulation of irrelevant high-frequency parameter drift that the standard QFWP cannot suppress at long input windows.

\paragraph{Summary of qualitative trends.}
Across all six tasks at $N=64$, three consistent conclusions emerge. First, GQKAN-QKANFWP achieves an early-epoch alignment with the ground truth that QFWP never matches on smooth-dynamics and quantum-dynamics tasks and reaches only partially on the NARMA family. Second, the seed-to-seed variability of GQKAN-QKANFWP remains negligible throughout training, whereas QFWP exhibits persistently wide variance bands that often expand beyond the train/test boundary. Third, and most importantly, extending training to 100 epochs does not close the performance gap: the QFWP mean predictions stabilize at qualitatively incorrect amplitudes or frequencies, indicating a representational limitation rather than an optimization delay. These extended-horizon observations reinforce the quantitative results reported in \cref{subsec:time_series} and provide direct visual evidence that the gated update rule and HQKAN-based fast-weight programming framework enable stable, accurate, and seed-robust long-window forecasting in regimes where the standard QFWP fails to converge.

\begin{figure}[!t]
\centering
\includegraphics[width=\textwidth]{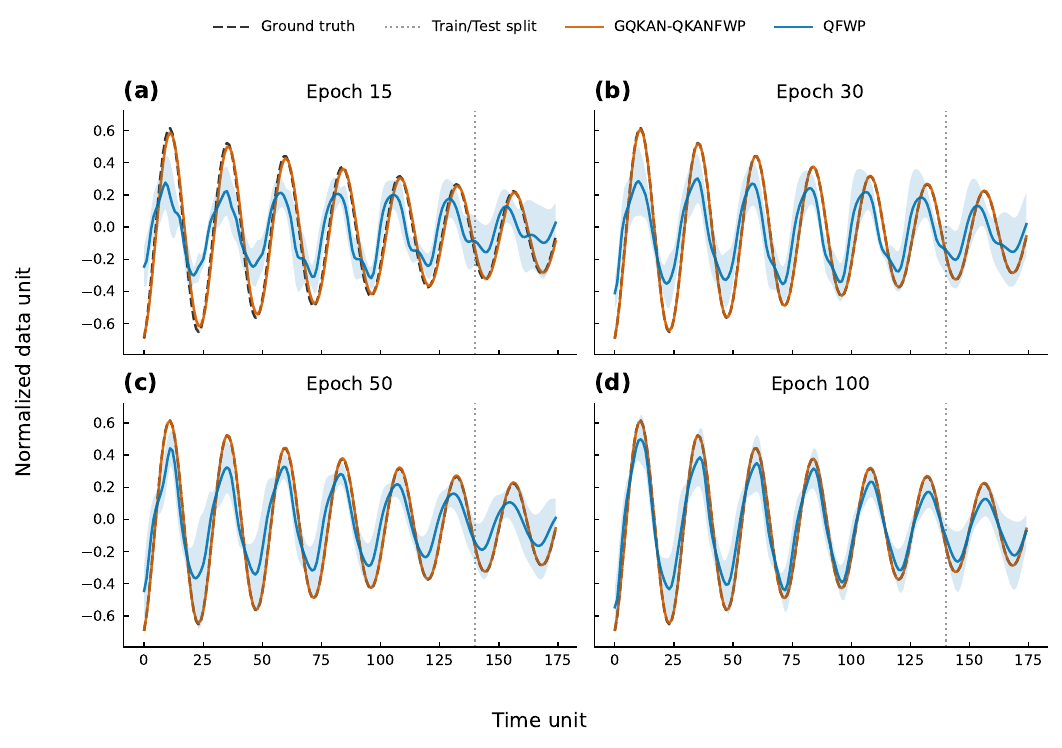}

\caption{\textbf{Forecasting performance on the Damped SHM dataset (Window-size N=64).} Panels \textbf{(a)}, \textbf{(b)}, \textbf{(c)}, and \textbf{(d)} illustrate the model predictions at training epochs 15, 30, 50, and 100, respectively. Solid lines denote the mean prediction across five independent random seed initializations for the proposed GQKAN-QKANFWP and the QFWP baseline. The shaded region represents the $\pm 1\sigma$ variance envelope for each model, demonstrating the model's stability across initializations. The ground truth dynamics are shown in dashed charcoal, and the vertical dotted line indicates the boundary between the training/validation phase and the unseen test set.}
\label{fig:shm}
\end{figure}

\begin{figure}[!t]
\centering
\includegraphics[width=\textwidth]{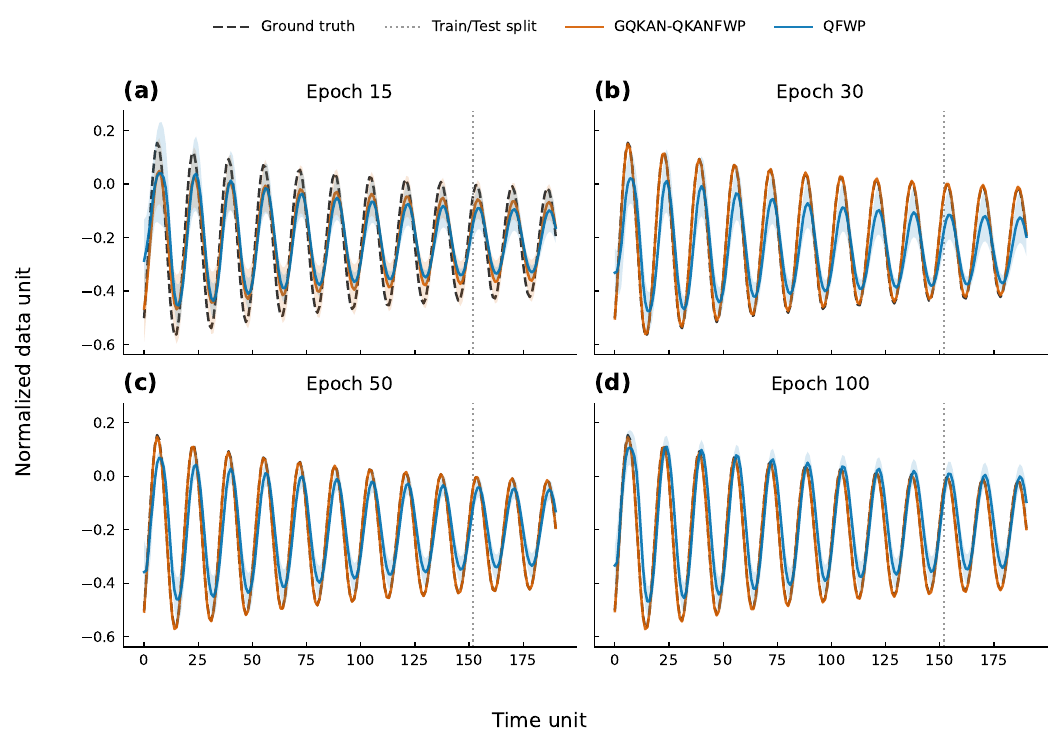}

\caption{\textbf{Forecasting performance on the Bessel function dataset (Window-size N=64).} Panels \textbf{(a)}, \textbf{(b)}, \textbf{(c)}, and \textbf{(d)} illustrate the model predictions at training epochs 15, 30, 50, and 100, respectively. Solid lines denote the mean prediction across five independent random seed initializations for the proposed GQKAN-QKANFWP and the QFWP baseline. The shaded region represents the $\pm 1\sigma$ variance envelope for each model, demonstrating the model's stability across initializations. The ground truth dynamics are shown in dashed charcoal, and the vertical dotted line indicates the boundary between the training/validation phase and the unseen test set.}
\label{fig:bessel}
\end{figure}

\begin{figure}[!t]
\centering
\includegraphics[width=\textwidth]{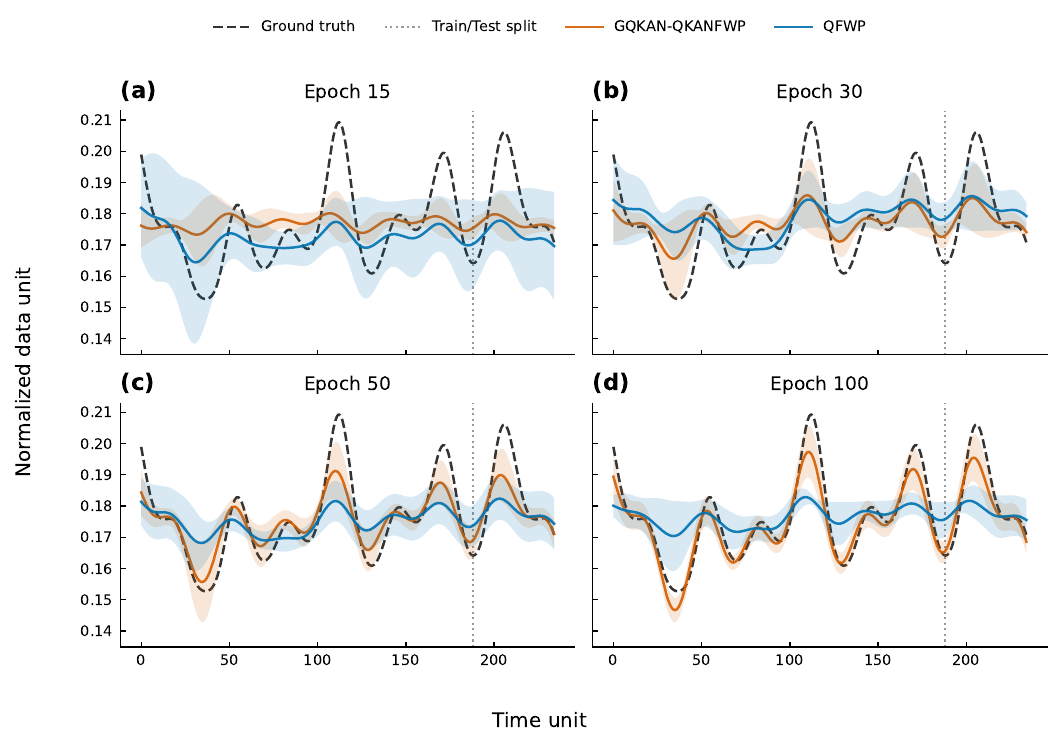}

\caption{\textbf{Forecasting performance on the NARMA5 dataset (Window-size N=64).} Panels \textbf{(a)}, \textbf{(b)}, \textbf{(c)}, and \textbf{(d)} illustrate the model predictions at training epochs 15, 30, 50, and 100, respectively. Solid lines denote the mean prediction across five independent random seed initializations for the proposed GQKAN-QKANFWP and the QFWP baseline. The shaded region represents the $\pm 1\sigma$ variance envelope for each model, demonstrating the model's stability across initializations. The ground truth dynamics are shown in dashed charcoal, and the vertical dotted line indicates the boundary between the training/validation phase and the unseen test set.}
\label{fig:narma5}
\end{figure}

\begin{figure}[!t]
\centering
\includegraphics[width=\textwidth]{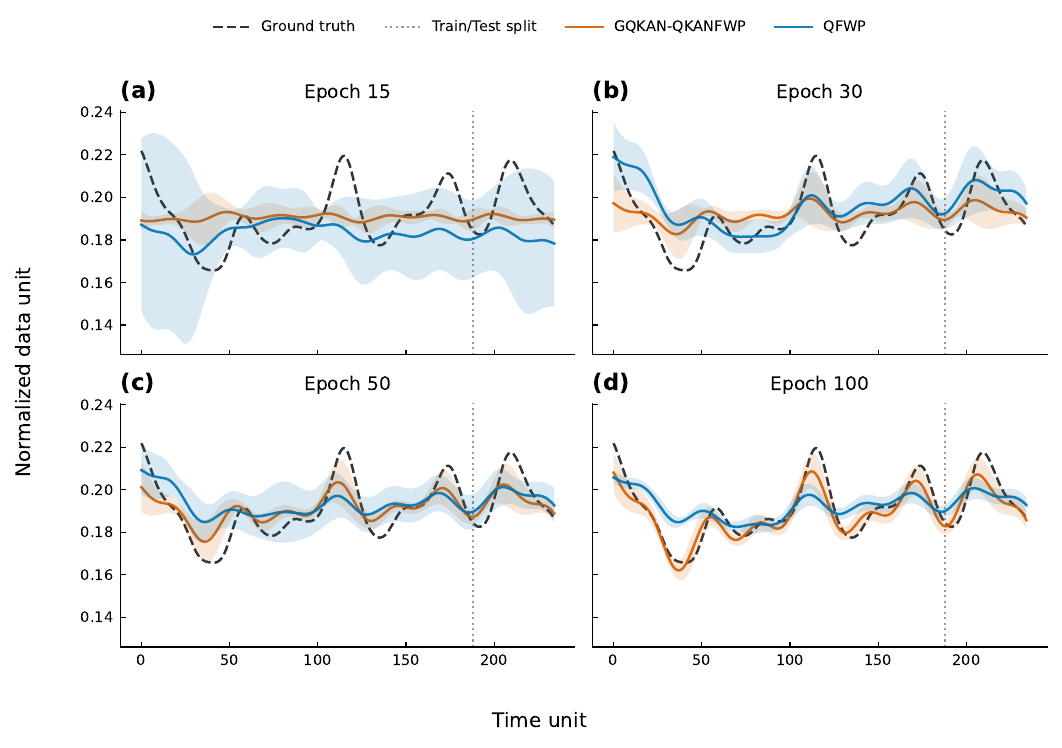}

\caption{\textbf{Forecasting performance on the NARMA10 dataset (Window-size N=64).} Panels \textbf{(a)}, \textbf{(b)}, \textbf{(c)}, and \textbf{(d)} illustrate the model predictions at training epochs 15, 30, 50, and 100, respectively. Solid lines denote the mean prediction across five independent random seed initializations for the proposed GQKAN-QKANFWP and the QFWP baseline. The shaded region represents the $\pm 1\sigma$ variance envelope for each model, demonstrating the model's stability across initializations. The ground truth dynamics are shown in dashed charcoal, and the vertical dotted line indicates the boundary between the training/validation phase and the unseen test set.}
\label{fig:narma10}
\end{figure}

\begin{figure}[!t]
\centering
\includegraphics[width=\textwidth]{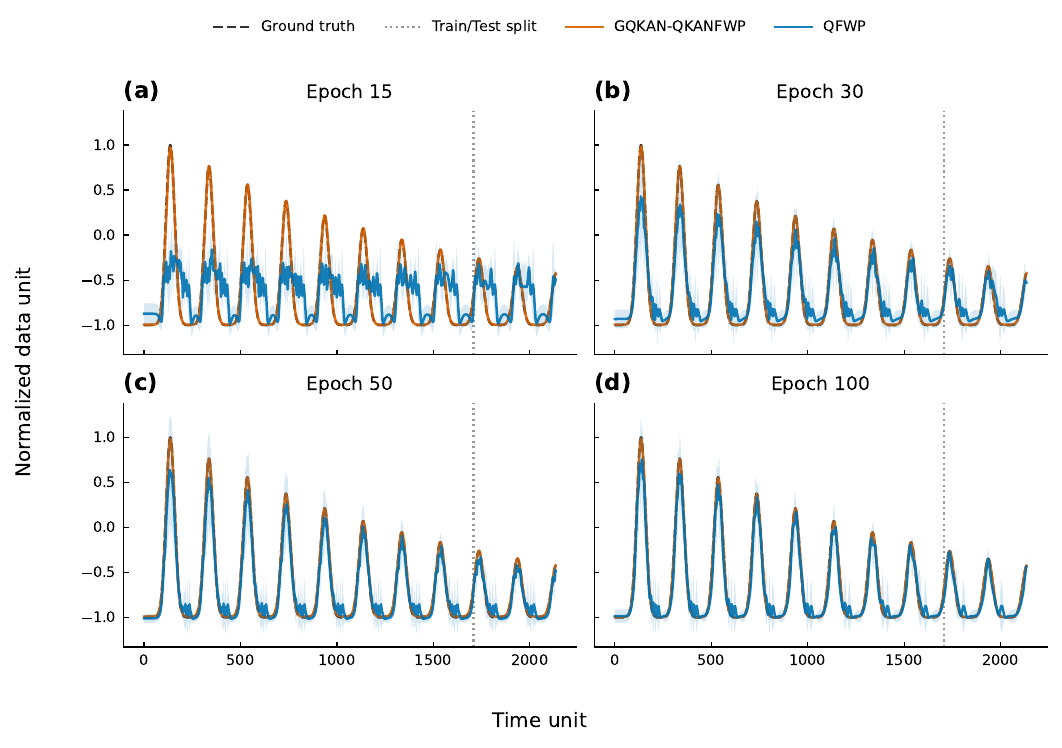}

\caption{\textbf{Forecasting performance on the Delayed Quantum Control dataset (Window-size N=64).} Panels \textbf{(a)}, \textbf{(b)}, \textbf{(c)}, and \textbf{(d)} illustrate the model predictions at training epochs 15, 30, 50, and 100, respectively. Solid lines denote the mean prediction across five independent random seed initializations for the proposed GQKAN-QKANFWP and the QFWP baseline. The shaded region represents the $\pm 1\sigma$ variance envelope for each model, demonstrating the model's stability across initializations. The ground truth dynamics are shown in dashed charcoal, and the vertical dotted line indicates the boundary between the training/validation phase and the unseen test set.}
\label{fig:DQC}
\end{figure}

\begin{figure}[!t]
\centering
\includegraphics[width=\textwidth]{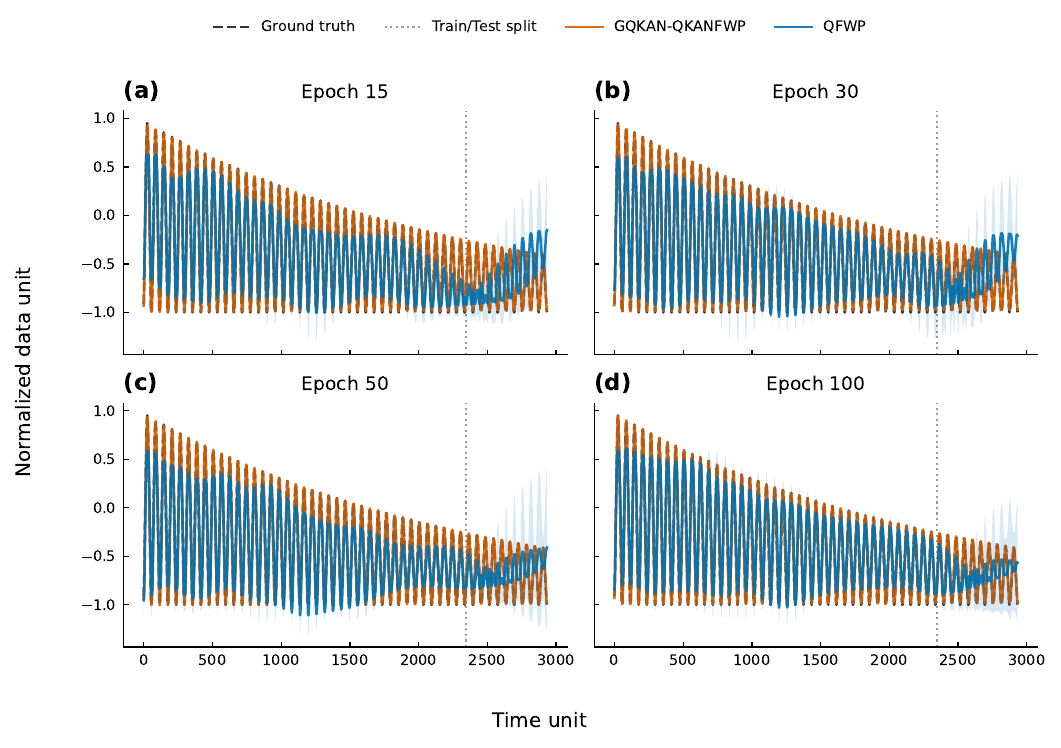}

\caption{\textbf{Forecasting performance on the Jaynes-Cummings dataset (Window-size N=64).} Panels \textbf{(a)}, \textbf{(b)}, \textbf{(c)}, and \textbf{(d)} illustrate the model predictions at training epochs 15, 30, 50, and 100, respectively. Solid lines denote the mean prediction across five independent random seed initializations for the proposed GQKAN-QKANFWP and the QFWP baseline. The shaded region represents the $\pm 1\sigma$ variance envelope for each model, demonstrating the model's stability across initializations. The ground truth dynamics are shown in dashed charcoal, and the vertical dotted line indicates the boundary between the training/validation phase and the unseen test set.}
\label{fig:JC}
\end{figure}

\end{document}